\documentclass[bj,authoryear,noshowframe]{imsart}
\startlocaldefs
\theoremstyle{plain}
\newtheorem{hyp}{Assumption}[section]
\newtheorem{theorem}{Theorem}[section]
\newtheorem{prop}[theorem]{Proposition}

\newtheorem{corollary}[theorem]{Corollary}

\theoremstyle{remark}

\newtheorem{example}[theorem]{Example}
\newtheorem{remark}[theorem]{Remark}
\newtheorem{rem}[theorem]{Remark}

\newlist{propenum}{enumerate}{1} % also creates a counter called 'propenumi'
\setlist[propenum]{label=(\roman*)}

\newcommand{\cc}{{\mathcal C}}

\newcommand{\ck}{{\mathcal K}}
\newcommand{\cK}{{\mathcal K}}

\newcommand{\cq}{{\mathcal Q}}

\newcommand{\cz}{{\mathcal Z}}

\newcommand{\sparse}{{ s }}

\newcommand{\E}{{\mathbb E}}

\newcommand{\N}{{\mathbb N}}
\renewcommand{\P}{{\mathbb P}}

\newcommand{\R}{{\mathbb R}}

\newcommand{\rd}{{\rm d}}

\newcommand{\inv}[1]{\mathop{\frac{1}{ #1}}\nolimits}
\newcommand{\expp}[1]{\mathop {\mathrm{e}^{ #1}}}

\newcommand{\dT}{\mathfrak{d}_T}

\newcommand{\dI}{\mathfrak{d}_\infty}
\newcommand{\dK}{\mathfrak{d}_\cK}

\newcommand{\DT}{\mathcal{V}_T}
\newcommand{\RT}{\rho_T}
\newcommand{\tD}{\tilde D}
\newcommand{\op}{\mathrm{op}}

\newcommand{\Var}{{\rm Var}}

\newcommand{\norm}[1]{{\left\lVert #1 \right\rVert}}

\newcommand{\idx}{z}
\newcommand{\spi}{\mathcal{Z}}
\endlocaldefs

\begin{document}

\begin{frontmatter}
\title{Simultaneous off-the-grid learning of mixtures issued from a continuous dictionary }
%\title{A sample article title with some additional note\thanksref{T1}}
\runtitle{Simultaneous off-the-grid learning of mixtures issued from a continuous dictionary }
%\thankstext{T1}{A sample of additional note to the title.}

\begin{aug}
\author[A]{\inits{C.}\fnms{Cristina}~\snm{Butucea}\ead[label=e1]{cristina.butucea@ensae.fr}}
\author[B]{\inits{J-F.}\fnms{Jean-François}~\snm{Delmas}\ead[label=e2]{jean-francois.delmas@enpc.fr}}
\author[C]{\inits{A.}\fnms{Anne}~\snm{Dutfoy}\ead[label=e3]{anne.dutfoy@edf.fr}}
\author[B,C]{\inits{C.}\fnms{Cl\'ement}~\snm{Hardy}\ead[label=e4]{clement.hardy@enpc.fr}}
%%%%%%%%%%%%%%%%%%%%%%%%%%%%%%%%%%%%%%%%%%%%%%
%% Addresses                                %%
%%%%%%%%%%%%%%%%%%%%%%%%%%%%%%%%%%%%%%%%%%%%%%
\address[A]{CREST, ENSAE, IP Paris, France\printead[presep={,\ }]{e1}}

\address[B]{CERMICS, \'{E}cole des Ponts, France\printead[presep={,\ }]{e2,e4}}
\address[C]{EDF R\&D, Palaiseau, France\printead[presep={,\ }]{e3}}
\end{aug}

\begin{abstract}
In this paper we observe a set, possibly a continuum, of signals corrupted by noise. Each signal is a finite mixture of an unknown number of features belonging to a continuous dictionary. The continuous dictionary is parametrized by a real non-linear parameter. We shall assume that the signals share an underlying structure by assuming that each signal has its active features included in a finite and sparse set. 
We formulate regularized optimization problem to estimate simultaneously the linear coefficients in the mixtures and the non-linear parameters of the features. The optimization problem is composed of a data fidelity term and a $(\ell_1,L^p)$-penalty. We call its solution the Group-Nonlinear-Lasso and provide high probability bounds on the prediction error using certificate functions. Following recent works on the geometry of off-the-grid methods, we show that such functions can be constructed provided the parameters of the active features are pairwise separated by a constant with respect to a Riemannian metric.
When the number of signals is finite and the noise is assumed Gaussian, we give refinements of our results for $p=1$ and $p=2$ using tail bounds on suprema of Gaussian and $\chi^2$ random processes. When $p=2$, our prediction error reaches the rates obtained by the Group-Lasso estimator in the multi-task linear regression model. Furthermore, for $p=2$ these  prediction rates are faster than for $p=1$ when all signals share most of the non-linear parameters. 
\end{abstract}

\begin{keyword}
\kwd{Continuous dictionary}
\kwd{group-nonlinear-lasso}
\kwd{interpolating certificates}
\kwd{mixture model}
\kwd{multi-task learning}
\kwd{non-linear regression model}
\kwd{off-the-grid methods}
\kwd{simultaneous recovery}
\kwd{sparse spike deconvolution}
\end{keyword}

\end{frontmatter}

%%%%%%%%%%%%%%%%%%%%%%%%%%%%%%%%%%%%%%%%%%%%%%
%%%% Main text entry area:

%%%%%%%%%%%%%%%%%%%%%%%%%%
\section{Introduction}

Observing repeatedly the same process is very frequent nowadays, due to the abundance of data in all fields. Multi-task learning considers the simultaneous analysis of multiple datasets and produces an estimator for each dataset. Datasets can be either discrete-time (e.g. regression models) or continuous-time in our context. We assume that they bring information on the same underlying structure.

We assume each process has a signal-plus-noise structure and that the signal is a mixture of features issued from a dictionary of smooth functions parametrized by some non-linear parameter (such as location, scale, etc.). Such mixtures can be seen e.g. in spectroscopy where each feature corresponds to a chemical component of the analyzed material, see \cite{butucea2021}. 

We are interested in recovering simultaneously the signals, i.e. the linear weights in the mixture and the non-linear parameters of the features, by minimizing a weighted prediction risk penalized by the sum of the total energy of the weights that each feature has through the collection of all processes. The prediction risk may put more weight on prescribed signals of interest. We give high probability bounds on the weighted prediction risk that are analogous to the case of multi-task discrete linear regression models.

\subsection{Model and method}

Let $H_T$ be a Hilbert space  where the
parameter  $T \in  \N$  accounts for  the amount of information in the model. The  Hilbert  space  $H_T$ is  endowed  with  the  scalar  product
$\langle \cdot, \cdot \rangle_T$ and the norm $\norm{\cdot}_T$. 

The observations are a collection of random elements of $H_T$ having a signal-plus-noise structure. The signal  part is a mixture (linear  combination) of at most $K$  smooth
features  $\varphi_T(\theta)$   belonging  to  $H_T$   and  continuously
parametrized by a real parameter  $\theta \in \Theta \subseteq \R$.  For example, consider the standardized Gaussian probability densities with mean $\theta$ or the Cauchy probability densities at location $\theta$.  We denote  by
$(\varphi_T(\theta), \theta \in \Theta)$  the continuous dictionary formed
by  all the  features.   
We consider
features $\varphi_T(\theta)$ that are non degenerate, {\it i.e.} for any $\theta\in \Theta$, $ \norm{\varphi_T (\theta)}_T$ is finite
and non-zero.
Let us define the normalized function $\phi_T(\theta)$ for $\theta\in \Theta$ and its multivariate counterpart $\Phi_T(\vartheta)$ for $\vartheta = \left ( \theta_1,\cdots,\theta_K\right ) \in \Theta^K$ by :
\begin{equation*}
{\phi_{T}(\theta)=\frac{\varphi_T(\theta)}{\norm{\varphi_T(\theta)}_T}} \quad \text{  and } \quad \Phi_{T}(\vartheta)= \begin{pmatrix}
\phi_{T}(\theta_{1})\\
\vdots \\
\phi_{T}(\theta_{K})
\end{pmatrix} \cdot
\end{equation*} 
Let
$(\Omega, \mathcal{G},\P)$  be a  probability space. We note  $W_T$ the
additional noise process defined on this  space and assumed to be almost
surely     an      element     of      $H_T$. 

An observation $Y$ writes:
$$
Y = \sum_{k=1}^K \beta_k^\star \cdot \phi_T(\theta_k^\star) + W_T \quad \text{in } H_T,
$$
where the row vector $\beta^\star = (\beta_1^\star,\ldots,\beta_K^\star)$ is $s-$sparse with non-zero coordinates in the set $S^\star$ (with cardinality $s$) and $\theta_k^\star$ belongs to $\Theta$ for all $k$ in $S^\star$. This can also be written as $Y = \beta^\star \cdot \Phi_T(\vartheta^\star) + W_T$ in $H_T$.

In this paper we consider a collection (either discrete or continuous) of such signals. We assume that all signals share $s$ features from our continuous dictionary. We describe first the discrete case and then the continuous case. We will give a general setup including both cases after the following examples.

\begin{example}[Discrete case]
    Let us assume that the process $Y$ has been observed repeatedly $n$ times. Thus, for $i$ in $\{1,\ldots, n\}$, we observe: 
    $$
    Y(i) = \sum_{k=1}^K B^\star_k(i) \cdot \phi_T(\theta_k^\star) + W_T(i), \quad \text{in }H_T.
    $$
    Let $L_T$ be the set of $H_T$-valued square integrable functions $f$ on $\{1,\ldots,n\}$ with:
    $$
    \|f\|_{L_T}^2:= \frac 1n \sum_{i=1}^n \|f(i)\|_T^2 < \infty.
    $$
    We endow $L_T$ with the scalar product $\langle f,g\rangle_{L_T} := \frac 1n \sum_{i=1}^n \langle f(i), g(i)\rangle_T$, for all $f,\, g$ in $L_T$. Thus, we obtain the Hilbert space $(L_T, \|\cdot \|_{L_T})$.

In our simultaneous analysis of the collection $Y$ of $n$ processes $Y(1),\ldots,Y(n)$, we assume that the matrix $B^\star$ with entries $B^\star_{i k}:= B^\star_k(i)$ has $s-$sparse column structure in the sense that the set:
$$
S^\star = \left\{ k \in \{1,\ldots,K\}: \frac 1n \sum_{i=1}^n \|B^\star_k(i)\|_T^2 \not = 0 \right\}
$$
has size $s$ with $1\leq s <K$. The model can be written:
\begin{equation}
    \label{discreteY}
Y = B^\star \cdot \Phi_T (\vartheta^\star) + W_T, \quad \text{in } L_T,
\end{equation}
where the set $S^\star$ and its size $s$, the vectors $B_k^\star$ and the values $\theta_k^\star$ for $k$ in $S^\star$ are unknown.

    This model generalizes the multi-task regression model (and the Group-Lasso model) to a design matrix whose columns are not fully observed, but are issued from the continuous dictionary of features at unknown values $\theta_k^\star$ for $k$ in $S^\star$. Note also that according to the choice of the Hilbert space $H_T$ we get different non-linear regression models. For example, if $H_T$ is $\mathbb{R}^m$ with its Euclidean norm we get a non-linear matrix regression model with unknown linear parameter $B^\star$ and unknown non-linear parameter $\vartheta^\star$ in the $n \times m$ design matrix $\Phi_T(\vartheta^\star)$. If $H_T$ is the space of square integrable functions on a compact measure set $\mathcal{T}$, we rewrite the model \eqref{discreteY} as the following multivariate functional data regression model,  for $i=1,\ldots,n$ and $t\in \mathcal{T}$:
    $$
    Y(i,t) = B^\star(i) \cdot \Phi_T (\vartheta^\star, t) + W_T(i,t).
    $$

    We may also need for practical reasons to associate to each observed process $Y(i)$ a score indicating, for example, the reliability of the method of acquisition of the observed data. In this context, one can add the information to the model by assigning weights $\nu(i)$ to each process $Y(i)$ and average the prediction risk accordingly. 
    In this context, we define on the space  $\mathcal{Z} = \{1,\ldots,n\}$  the measure $\nu$ and $L_T = L^2(\nu, H_T)$ is the space of $H_T-$valued functions $f$ such that:
    $$
    \|f\|^2_{L_T} := \int_{\mathcal{Z}} \|f(i)\|^2_T \, \rd\nu(i) < \infty.
    $$
    \end{example}

\begin{example}[Continuous case]
Let us assume now that the process $Y$ is observed continuously at $z$ belonging to some set $\mathcal{Z}$:
$$
Y(z) = \sum_{k=1}^K B_k^\star(z) \cdot \phi_T(\theta_k^\star) +W_T(z), \quad \text{in } H_T,
$$
where the set $S^\star$ of indices $k$ such that $B^\star_k$ is non-zero, the values $B_k^\star(z)$ and $\theta_k^\star$ for $k $ in $S^\star$ are unknown. 
Such models are known as "function-on-scalar" models, referring to regression models where the linear coefficients depend on a time or spatial continuous parameter, see \cite{Barber17}.
    
Let $(\spi, \mathcal{F},\nu) $ be any measure space such that $0<\nu(\spi) < + \infty$; we can take $\spi$ as a compact interval of $\R$ and $\nu$ as the Lebesgue measure on $\spi$. Here, $L_T$ denotes the set of $H_T-$valued square  integrable functions $f$ on $\spi$ with:
$$
\|f\|_{L_T}^2: = \int_\spi \|f(z)\|_T^2 \, \rd\nu(z) < \infty.
$$
Again we assume that the functional linear parameters share a sparse structure: the unknown set $S^\star$, which  is then given by $\{k\in \{1, \ldots, K\}\, \colon\, \|B^\star_k\|_{L_T}^2 \neq 0\}$, is sparse with cardinality $s\ll K $.

    Hence, we generalize the ``function-on-scalar'' models that have many applications ({\it e.g.} in genomics, see \cite{Barber17}) by allowing the design matrix to be parametrized. 
\end{example}

\medskip

In all generality, let  $(\spi,\mathcal{F},\nu)$ be  a measure  space with  $\nu$ a  finite
positive non-zero  measure. We consider
the space  $L_T = L^2(\nu, H_T)$,  the set of
$H_T$-valued strongly measurable functions $f$  defined on $(\spi,\mathcal{F},\nu)$ such that
$\norm{f}_{L_T}= \sqrt{\int_\spi \norm{f(z)}_T^2 \, \nu  (\rd z)}$ is finite. We then endow $L_T$ with a scalar product  noted $\left \langle \cdot, \cdot \right \rangle_{L_T}$ defined for any $f,g \in L_T$ by :
\[
\langle f,g \rangle_{L_T}= \int_{\spi} \left \langle f(\idx), g(\idx)\right \rangle_{T} \nu(\rd \idx) .
\]
The norm $\norm{\cdot}_{L_T}$ is the natural norm associated with the scalar product and $(L_T,\norm{\cdot}_{L_T}) $ is an Hilbert space, see \cite[Section IV]{diestel}.
For $p\in [1, +\infty )$, we write $L^p(\nu, \R^K)$ for the space of $\R^K$-valued  measurable function $f$  defined on $(\spi,\mathcal{F},\nu)$ such that
\[
\norm{f}_{L^p(\nu, \R^K)}=\left( \int_\spi \norm{f(z)}_{\ell_2}^p \, \nu  (\rd z) \right)^{\frac{1}{p}}
\]
is finite, where
$\norm{\cdot}_{\ell_2} $ is the usual Euclidean norm  on $\R^K$.
We simply write $L^p(\nu)$ for  $L^p(\nu, \R)$.

We observe a random element $Y$ of
the Hilbert space  $L_T$. We   consider  the   model  with   unknown  parameters   $B^{\star}$  in
$L^2(\nu,\R^{K})$ and $\vartheta^{\star}$ in $\Theta^K$:
\begin{equation}
\label{eq:model}
Y = B^{\star}\Phi_{T}(\vartheta^{\star}) + W_T \quad \text{ in } L_T.
\end{equation}
Here, we assume that the mapping $B^{\star} : \spi \rightarrow \R^K$ is $s-$ sparse that is,
\[
1 \leq s < K \quad\text{with}\quad s = \operatorname{Card}(S^\star) \quad\text{and}\quad S^\star = \{k\in \{1, \ldots, K\}\, \colon\, \norm{B_{k}^{\star}}_{L^2(\nu)} \neq 0 \, \}.
\]
The set $S^\star$ and the parameters $B^\star$ and $\vartheta^\star$ are unknown. Thus, the sparsity $s$ is unknown, but an upper bound $K$ on this value is supposed available. The value $K$ is used as a maximal size of our parameters and to write the optimization problem that we solve here after in order to build estimators, but it does not appear in the rates we obtain later. 
In order to perform signal reconstruction, we are interested in recovering the sparse mapping $B^{\star}$ restricted to its support $S^\star$, that is $B^\star_{S^\star}$, and the associated parameters $\vartheta^\star_{S^\star}$ of the nonlinear parametric functions involved in the mixture model.

We remark that the model \eqref{eq:model} is an extension of the model described in \cite{butucea22}, as the latter amounts to taking $\spi$ a singleton (or $\nu$ a Dirac measure). We gain in generality by letting the measure  $\nu$ be any finite positive non-zero measure on $\mathcal{Z}$, see Section~\ref{sec:contrib}
for further comments. By doing so, the observation $(Y(\idx), \idx \in \spi)$ can be applied {\it e.g.} to longitudinal data and to multiple mixture models.

\medskip

In order to recover the sparse mapping $B^{\star}$ as well as the associated parameters $\vartheta_{S^{\star}}^\star$ (up to a permutation) we solve a regularized optimization problem, that we call Group-Nonlinear-Lasso, with a real tuning parameter $\kappa>0$ and $p \in [1,2]$:

\begin{equation}
\label{eq:generalized_lasso}
(\hat{B},\hat{\vartheta}) \in \underset{B \in
	L^2(\nu,\mathbb{R}^{K}),
	\vartheta \in \Theta_{T}^K}{\text{argmin}} \quad \frac{1}{2 \nu(\spi)} \norm{Y -
	B\Phi_{T}(\vartheta)}_{L_T}^{2} +\kappa \norm{B}_{\ell_1, L^p(\nu)},
\end{equation}
where for $z \mapsto B(z)=(B_1(z), \ldots, B_K(z))$ in $L^2(\nu,\mathbb{R}^{K})$:
\begin{equation*}
\norm{B}_{\ell_1, L^p(\nu)} = \sum\limits_{k=1}^{K}\norm{B_k}_{L^p(\nu)}.
\end{equation*}
The  set  $\Theta_T$  on  which  the  optimization  of  the  non-linear
parameters is  performed is required  to be  a compact interval  and the
function  $\Phi_T$  is continuous.  When  $\mathcal{Z}$  is finite,  the
existence  of  at  least  a   solution  is  therefore  guaranteed for any $p$ in $[1,2]$.  When
$\mathcal{Z}$ is infinite and $p $ in  $(1,2]$, we may use the following
result whose proof (based on the reflexivity of $L^p(\nu)$ which is not valid for $p=1$) is given in Section~E.1 of the supplementary material \cite{butuceaSupplement}.

\begin{prop}
	\label{prop:existence_solution}
	Let $p \in (1,2]$. Assume that the function $\theta \mapsto
	\phi_T(\theta)$ is continuous. Then, the minimization problem
	\eqref{eq:generalized_lasso} over 	$L^2(\spi,\mathbb{R}^{K}) \times
	\Theta_{T}^K$, where $\Theta_T$ is a compact interval of $\R$, admits at
	least one solution. 
\end{prop}
The estimator $\hat \vartheta$ defined in \eqref{eq:generalized_lasso}  is called off-the-grid as it does not depend on any discretization scheme applied to the parameter space $\Theta$. This approach differs  from previous works in which the parameter space is discretized and the dictionary used to approximate the signals is therefore finite, see \cite{tang2013sparse} in this direction.

In  this paper,  we aim  at quantifying  the quality  of the  prediction of
$B^\star  \Phi(\vartheta^\star)$ by  $\hat B  \Phi(\hat \vartheta)$  for
$\hat B$ and $\hat  \vartheta$ given by \eqref{eq:generalized_lasso}, by
providing an upper bound with high probability of the squared prediction
error:
\begin{equation}
\label{eq:prediction_error}
\hat{R}_T^2 = \frac{1}{\nu(\mathcal{Z})}\norm{B^\star \Phi(\vartheta^\star) - \hat B \Phi(\hat \vartheta)}_{L_T}^2.
\end{equation}

To understand the  normalization by $\nu(\spi)$, consider the previous example of a finite collection of processes:   $\mathcal{Z} = \{1,\cdots,n\}$ and $\nu$ the
	counting  measure  $\sum_{i=1}^n   \delta_i$.  Assume  the  $n$
	observations  belong  to the  Hilbert  space  $H_T=L^2(\lambda)$ for  some
	measure $\lambda$ (either  discrete or continuous) on  the Borel sigma
	field of $\R$. In this case, the squared prediction error becomes:
	\begin{equation}
 \label{eq:RT2=}
	\hat{R}_T^2 = \frac{1}{n} \sum_{i=1}^n \norm{B^\star(i)
		\Phi(\vartheta^\star) - \hat B(i) \Phi(\hat
		\vartheta)}_{L^2(\lambda)}^2. 
	\end{equation}

\subsection{Previous work}
\label{sec:previouswork}

Reconstructing from  observations (that are discrete or continuous-time processes) signals that are linear combinations of features belonging to a continuous dictionary $(\varphi(\theta), \, \theta \in \Theta)$  has applications in many fields such as super-resolution (\cite{candes2014towards}), spike deconvolution (\cite{duval2015exact}), microscopy (\cite{denoyelle2019sliding}) or  spectroscopy (\cite{butucea2021}).

Most often, the  Hilbert space $H_T$, to which  the observations belong,
is assumed to  be of finite dimension and the  dictionary of features is
assumed finite of  size $K$.  Over the past two  decades, the problem of
retrieving  a  sparse  vector  in  the  framework  of  high  dimensional
regression  models ($K  \gg  \operatorname{dim}(H_T)$)  has generated  a
large     number      of     works     (\cite{tibshirani1996regression},
\cite{bickel2009simultaneous},                 \cite{bunea2007sparsity},
\cite{candes2007dantzig},  \cite{buhlmann2011statistics} and  references
therein).     The    celebrated    Lasso   estimator,    popularized   by
\cite{tibshirani1996regression} and  defined by an  optimization problem
composed  of a  data  fidelity  term and  a  $\ell_1$  penalty, has  been
extensively studied  and has  proven to be  efficient. In  addition, its
convex   formulation  makes   its   resolution  easy   to  handle   (see
\cite{Beck2009}     for    a     resolution    via     fast    iterative
shrinkage-thresholding   algorithms).   Prediction  error   bounds   and
estimation  bounds  with   respect  to  the  $\ell_2$   norm  have  been
established  for the  Lasso under  coherence assumptions  on the  finite
dictionary. We refer to \cite{van2009conditions}  for an overview of the
coherence assumptions.  It turns out  that these rates have  been proven
mini-max optimal in \cite{MR2882274}. This means that one cannot find any
estimator that achieves faster rates in expected value  when estimating the worst possible parameters.

The prediction error bounds obtained  for sparse high-dimensional linear
models  encompass the  finite dictionary  setting. We  consider in  this
paper  continuous  dictionaries.  As   a  consequence,  the  problem  of
recontruction is highly non-linear. 

A line of work has emerged around the reconstruction of signals that are mixtures of continuously parametrized  features by solving a regularized minimization problem over a space of measures. Indeed, one can readily notice that a mixture of non-linear features $\sum_{k \in S^\star} \beta_k^\star \phi(\theta_k^\star)$ can be written as the application of the linear functional $\mu \mapsto \int \phi(\theta) \mu(\rd \theta) $ to the atomic measure  $\mu^\star = \sum_{k \in S^\star} \beta_k^\star \delta_{\theta_k^\star}$, where $\delta_x$ denotes a Dirac measure located in $x$.  The Beurling Lasso (or BLasso) introduced in \cite{de2012exact} has proven to be efficient to retrieve a sparse measure from its images through  linear functionals.
We stress that when  $\operatorname{dim}(H_T) < +\infty$, there exists a solution to the BLasso made up of at most ${\rm dim}(H_T)$ Dirac measures. We refer to \cite{boyer2019representer} and \cite{duval21} for proofs of this result. For this reason, the BLasso has been used as a counterpart of the classical Lasso for continuous dictionaries. We remark that when $H_T$ is infinite dimensional the BLasso optimization problem over the space of measures may not have a solution which is an atomic measure.  It makes its solutions difficult to interpret in our context. That is why we prefer in this paper to assume a bound $K$ on the unknown number of features $s$ in order to formulate \eqref{eq:model} and to solve a different optimization problem \eqref{eq:generalized_lasso} producing an atomic measure as a solution. When only one element of $H_T$ is observed (\emph{i.e.} $\spi$ is reduced to a singleton and $\nu$ is a Dirac measure), this formulation is equivalent to that of the BLasso restricted to the set of atomic measures of at most $K$ atoms. Efficient numerical methods to solve this problem are available such as modifications of the Frank-Wolfe algorithm (\cite{denoyelle2019sliding}, \cite{boyd2017alternating}) or the Conic Gradient Particle Descent (\cite{chizat2021sparse}). We stress that these methods proceed by seeking a solution that is atomic.

It has been shown that under the assumption of the existence of certificate functions, the BLasso  retrieves the exact number of features in a small noise regime (\cite{candes2014towards} for a specific dictionary  and \cite{duval2015exact} in a more general framework). Regarding prediction error bounds,  the research  has first focused on mixtures of features issued from a dictionary of complex exponentials parametrized by their frequencies. Much progress has been done in super-resolution using the BLasso with  this specific dictionary, see \cite{candes2014towards}, \cite{candes2013super} in this direction. In \cite{boyer2017adapting}, the authors showed that the prediction error of the BLasso estimator in this specific case almost reached that of the Lasso estimator provided the frequencies are well separated. They adapted previous results  from \cite{bhaskar2013atomic} and \cite{tang2014near} for atomic norm denoising  and they extended them to a more general case where the noise level is unknown and needs to be estimated.
The authors of the present paper considered  in \cite{butucea22} the model \eqref{eq:model} when only one signal is considered ($\spi$ is a singleton and  $\nu$ is a Dirac measure) and showed  that when the one-dimensional non-linear parameters of the features are well separated, one can build estimators that lead to a nearly optimal prediction error bound. By nearly optimal, we mean that the prediction error bound obtained in \cite{butucea22} is of the same order (up to a logarithmic factor) as the minimax bounds obtained in the finite dictionary setting where only linear coefficients are to be retrieved. The result covers a large variety of dictionaries and noises. Let us specify that the separation is expressed with respect to a Riemannian metric following the insightful work of \cite{poon2018geometry}.

\subsection{Contributions}\label{sec:contrib}
We extend the work of \cite{butucea22} to encompass the case of multiple (a discrete or continuous collection of) mixture models.
In this prior work, we studied a method to reconstruct efficiently a single signal and illustrate it for various examples of observation spaces, dictionaries and noise settings. Here, our goal is to reconstruct more generally a set (possibly a continuum) of signals. Of course, when dealing with a finite set of signals, one could reconstruct each signal individually using the method employed in our previous work. However, we show here that the simultaneous reconstruction with $p=2$ outperforms individual reconstruction when  all signals share most of the non-linear parameters. To obtain this enhancement, we introduce an optimization problem with a mixed-norm penalty, develop novel certificates, derive tail bound inequalities for the supremum of $\chi^2$ processes, and substantially expand the proof presented in \cite{butucea22} that only covers the case where the measure $\nu$ is a Dirac distribution.

We let here $\nu$ be any finite positive non-zero measure. 
In the framework of multiple high dimensional linear regressions  $(\ell_1,\ell_p)$-mixed norm penalties have been used to retrieve sparsity patterns among the signals. These penalties influence globally the estimations of the signals $(B(i)\Phi(\vartheta^\star), \,i \in \spi)$. Let us mention the $(\ell_1,\ell_2)$ mixed norm,  used to define the Group-Lasso estimator introduced in \cite{Yuan2006} and that has received significant attention since then (see, \cite{Nardi08}, \cite{Bach08}, \cite{Chesneau08}, \cite{Huang10}). It was shown in \cite{lounici2011oracle} that the reconstruction of signals via the Group-Lasso estimator outperfoms the reconstruction using the Lasso estimator  when the signals share some sparsity pattern. Let us mention the work of \cite{liu2008} that provides consistency results and prediction error convergence rates for the general case $(\ell_1,\ell_p)$ with $p \in [1, +\infty]$. Estimators obtained from regularized problems via mixed norms have been studied in the context of high dimensional multiple linear regression models but little has been done for the non-linear extension considered in \eqref{eq:model}. It is therefore natural to find counterpart estimators for the setting of continuous dictionaries.
Let us highlight the work of \cite{golbabaee2020off} in which an extension of the BLasso has been proposed in order to address multiple mixture models. The authors extended the work of \cite{duval2015exact} to show exact support recovery results in the small noise regime. They used a penalty that is a convex  combination of mixed  norms on measures.  We remark that when applied to atomic measures these  norms reduce to the $(\ell_1,\ell_1)$ and $(\ell_1,\ell_2)$  norms on the weights of the Dirac measures.

In this paper, we prove a high-probability upper bound on the prediction error for estimators issued from an optimization problem regularized by a mixed norm $(\ell_1,L^p(\nu))$ with $p\in [1,2]$ for a wide variety of dictionaries in the general framework where $\nu$ can be any finite positive measure. We give refinements of this result when the noise is assumed Gaussian and when the measure $\nu$ is discrete. These refined bounds on the prediction error use tail bounds on suprema of Gaussian and $\chi^2$ processes.  Our results rely on the existence of certificate functions, see Section \ref{sec:certificates}. We also give sufficient conditions for their construction.

\subsection{Group-Nonlinear-Lasso vs. Group-Lasso on a grid}
Our main objective is to  reconstruct signals that are linear combinations of features continuously parametrized.  This problem has been long handled by discretizing the parameter  space $\Theta$  and using a finite dictionary to approximate the signals as suggested in \cite{tang2013sparse}. In this way, the problem is reduced to a (high-dimensional) linear model which has been extensively studied in the literature.  However,  recent papers have advocated  that taking a finite subfamily of a continuous  dictionary and using a Lasso  estimator to  retrieve  the  linear  coefficients  of the  approximating mixture  lead  to  some issues.  In particular,  the number  of active  features in  the mixture tends to be overestimated, see \cite{duval2017thingrid} in the context of reconstructing a single signal. This phenomenon can also be observed in our more general multi-task setting. 

To illustrate this, we conduct a short numerical experiment. We consider a scenario where we have $n = 100$ noisy signals observed at 100 equally spaced points between -10 and 10; all signals share an underlying structure that consists of two spikes with unknown locations and varying amplitudes (refer to Figure \ref{Fig:signal} for a visual representation of such a signal).
\begin{figure}[!h]
	\centering
	\includegraphics[width=0.4\linewidth]{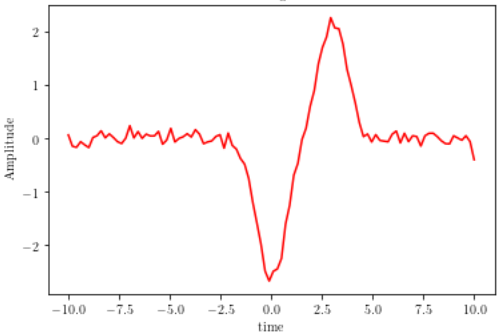}
	\caption{Signal in $H_T = \R^{T}$ with $T=100$, mixture of two Gaussian-shaped spikes with $\theta_1^\star = 0$ and $\theta_2^\star = 3$ and amplitudes in [-10,10] uniformly distributed, corrupted by i.i.d. centered Gaussian r. v. with $\sigma = 0.1$.}
	\label{Fig:signal}
\end{figure}
In Figure \ref{Fig:comparison}, we compare the performance of the Group-Nonlinear-Lasso (with $p=2$ and $K=50$) to that of the Group-Lasso in reconstructing these signals. In the Group-Lasso approach we give three examples of regular grids of non-linear parameters with different grid steps and such that the two spikes are always located at half distance of two consecutive points on the grid. 

The Group-Nonlinear-Lasso outperforms the Group-Lasso in terms of prediction error, regardless of the penalty strength. In addition, the Group-Nonlinear-Lasso accurately identifies the two spikes, while the Group-Lasso approach incorrectly detects four spikes, even when we refine the grid.
 
\begin{figure}[!h]
	\centering
	\begin{minipage}{.4\textwidth}
		\centering
		\includegraphics[width=1\linewidth]{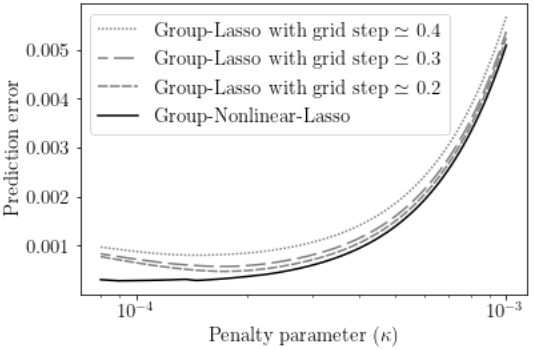}	
	\end{minipage}%
	\begin{minipage}{.39\textwidth}
		\centering
		\includegraphics[width=1\linewidth]{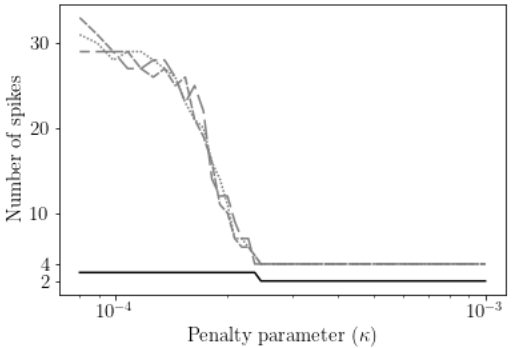}
	\end{minipage}%
	\caption{Prediction error $\hat R_T^2=\norm{Y-\hat Y}_{\ell_2}^2/(n T)$ given in~\eqref{eq:RT2=}, with  $\hat Y$ denoting the reconstructed signals, and  number of spikes obtained with  the Group-Nonlinear-Lasso and the Group-Lasso approaches. These quantities are represented as functions of the penalty parameter $\kappa$. }
	\label{Fig:comparison}
\end{figure}
All the figures contained in this section can be reproduced using the code available online at https://github.com/ClementHardy/PySFW.

\subsection{Organization of the paper and notation}
In Section \ref{sec:assumptions}, we formulate assumptions on the model and set some definitions.  Section \ref{sec:main_result} presents the main results of this paper. We start by giving a high probability upper bound on the prediction error in the general case where the measure $\nu$  can be any finite measure. Then, we give refinements of this result when the measure $\nu$ is a finite weighted sum of Dirac measures and the noise process is assumed Gaussian. In Section \ref{sec:certificates}, we present the assumptions on certificate functions that are used to state the high probability upper bound on the prediction error in Section~\ref{sec:assump_certificates}. We give in Section \ref{sec:construction_certificates} sufficient conditions to construct such functions. Section \ref{sec:proofsmaintheorem} is dedicated to the proof of the high probability upper bound on the prediction error in the most general framework.

{\bf Notation} We shall use for convenience the notation $\lesssim$ and write for two real quantities $a$ and $b$,  $a \lesssim b$ if there exists a positive finite constant $C$ independent of the parameters $s, K, T$ and the measure $\nu$ such that $a \leq C \, b$. 
We also write for two quantities $a,b$ that $a \asymp b$ if $a \lesssim b$ and $b \lesssim a$.

We write $a\wedge b =\min(a,b)$ and $a\vee b=\max(a,b)$.

%%%%%%%%%%%%%%%%%%%%%%%%
\section{Assumptions on the model}
\label{sec:assumptions}
%%%%%%%%%%%%%%%%%%%%%%%%

\subsection{Regularity and non-degeneracy assumptions on the features}
Let $T \in \N$ be a fixed  parameter. The features $(\varphi_T(\theta), \theta \in \Theta)$ that form a continuous dictionary are elements of the Hilbert space $(H_T, \left \langle \cdot, \cdot \right \rangle_T)$. We shall integrate and differentiate those features with respect to their one-dimensional parameter belonging to the interval $\Theta$ of $\R$. To do so, we shall use the notions of Bochner integral and Fréchet derivative. 
We  recall that for any function $f: \Theta \mapsto H_T$  differentiable    at
$\theta\in \Theta$, we have for all $g\in H_T$ that:
\begin{equation*}
\label{eq:deriv}
\partial_{\theta}\left \langle f(\theta),  g\right \rangle_T =
\left \langle \partial_\theta f(\theta),  g \right \rangle_T.
\end{equation*}
In addition, if $f$ is  Bochner integrable on $\Theta$,  then for all $g\in  H_T$, we have
that:
\begin{equation*}
\label{eq:integ}
\int_\Theta \langle f(\theta), g \rangle_T \, \rd \theta=  \langle
\int_\Theta  f(\theta)\,\rd \theta, g \rangle_T. 
\end{equation*}
We shall require the features to satisfy the following regularity assumption.
\begin{hyp}[Smoothness of $\varphi_T$]
	\label{hyp:reg-f} 
	We assume that the function $\varphi_T: \Theta \rightarrow H_T$
	is of class    $\cc^3$ and
	$\norm{\varphi_T   (\theta)}_T    >   0$   on    $\Theta$.
\end{hyp}
Assume that Assumption \ref{hyp:reg-f} holds. Recall that $\phi_{T}(\theta)= \varphi_T(\theta)/\norm{\varphi_T(\theta)}_T$ for all $\theta \in \Theta$. We define the continuous function:
\begin{equation}
\label{def:g_T}
g_T(\theta)= \norm{\partial_\theta
	\phi_T(\theta)}_T^2.
\end{equation}
It will be convenient to assume the non-degeneracy of the function $g_T$.
\begin{hyp}[Positivity of $g_T$]
	\label{hyp:g>0}
	Assumption \ref{hyp:reg-f} holds and we have	$g_T>0$ on $ \Theta$.
\end{hyp}

One can easily show that features are non-degenerate by checking that for any $\theta \in \Theta$ the elements 
$\varphi_T(\theta)$  and $\partial_\theta\varphi_T(\theta)$ of $H_T$  are
linearly   independent, see \cite[Lemma 3.1]{butucea22} in this direction.

\subsection{The kernel and its Riemannian derivatives}
\label{sec:riemann}
In this section, we introduce a function on $\Theta^2$, called kernel, that will quantify the correlation between two features in the dictionary. We shall derive from this kernel a Riemannian metric on the parameter space $\Theta$ following \cite{poon2018geometry}. This metric  will be in particular invariant to a reparametrization of the parameter space.

\subsubsection{Kernel space and associated Riemannian metric}

We call kernel  a real-valued
function defined on $\Theta^2$. Let $\cK$ be a symmetric kernel of class
$\cc^2$ such  that the  function $g_\cK$  defined on the one-dimensional and
connected set
$\Theta$ by:
\begin{equation}
\label{eq:def-gK}
g_\cK(\theta)= \partial^2_{x,y}
\cK( \theta,\theta)
\end{equation}
is positive and locally  bounded, where $\partial_x$
(resp. $\partial_y$) denotes the usual  derivative with respect to the
first (resp. second) variable. 

We derive from the kernel $\ck$ the  metric $\dK(\theta,\theta')$   between
$\theta, \theta'\in \Theta$ by:
\begin{equation}
\label{eq:def-Riemann-dist-v2}
\dK(\theta,\theta')=  |G_\cK(\theta) -G_\cK(\theta')|,
\end{equation}
where  $G_\cK$ is a primitive of $\sqrt{g_\cK}$. 

We need to differentiate the kernel $\ck$ on the manifold $(\Theta,g_\ck)$. We use the covariant derivatives  that generalize the classical directional derivative of vector fields on a manifold. Since we only consider the case of a one-dimensional parameter space, the covariant derivatives reduce to simple expressions.

For a real-valued function $F$ defined on $\Theta^2$, we say that $F$ is
of class $\cc^{0,  0}$ on $\Theta^2$ if it is  continuous on $\Theta^2$,
and of class $\cc^{i, j}$ on $\Theta^2$,  with $i, j\in \N$, as soon as:
$F$  is of  class  $\cc^{0, 0}$,  and  if $i\geq  1$  then the  function
$\theta\mapsto F(\theta, \theta')$  is of class $\cc^i$  on $\Theta$ and
its derivative $\partial_x F$ is  of class $\cc^{i-1, j}$ on $\Theta^2$,
and if $j\geq 1$ the function $\theta' \mapsto F(\theta, \theta')$ is of
class $\cc^j$ on $\Theta$ and its  derivative $\partial_y F$ is of class
$\cc^{i, j-1}$ on $\Theta^2$.
For a  real-valued  symmetric function $F$ defined
on $\Theta^2$ of class $\cc^{i,j}$, we define the
covariant  derivatives $D_{i,j;\cK}[F]$ of
order  $(i,j)\in \N^2$    recursively by  $D_{0,0;\cK}[F] =  F$ and  for
$i,j\in \N$, assuming that $g_\cK$ is of class $\cc^{i\vee j}$, and $\theta, \theta'\in \Theta$:
\begin{equation}
\label{eq:def-cov-deriv-2}
D_{i+1,j;\cK}[F](\theta,\theta')
= g_\cK(\theta)^{\frac{i}{2}}\partial_{\theta} \left(
\frac{D_{i,j;\cK}[F](\theta, \theta')
}{g_\cK(\theta)^{\frac{i}{2}}}\right)
\quad\text{and}\quad
D_{i,j;\cK}[F](\theta, \theta')=D_{j,i;\cK}[F](\theta', \theta). 
\end{equation}
In particular, we have  $D_{1, 0; \cK}=\partial_x
F$, $D_{0, 1; \cK}=\partial_y
F$ and $D_{1,1; \cK}=\partial^2_{xy}  F$. 
We shall also consider the following modification of the covariant
derivative, for $i,j\in \N$:
\begin{equation}
\label{eq:def-tD2}
\tD_{i,j;\cK}[F](\theta, \theta')=\frac{D_{i,j; \cK}[F](\theta,
	\theta')}{g_\cK(\theta)^{i/2}\, g_\cK(\theta')^{j/2} }\cdot
\end{equation}
We have $\tD_{1,0; \cK} \circ \tD_{0,1; \cK} =\tD_{0,1; \cK} \circ
\tD_{1,0; \cK} $ and  for
$i,j\in \N$, assuming that $g_\cK$ is of class $\cc^{i \vee j}$:
\begin{equation*}
\tD_{i,j; \cK}=\left(\tD_{1,0; \cK}\right)^i\circ \left(\tD_{0,1; \cK}\right)^j.
\end{equation*}
The definitions of covariant derivatives and their modifications cover the case of 1-dimensional functions defined on $\Theta$. For any smooth function $f$ defined on $\Theta$, we shall note $D_{i;\cK}[f]$ (resp. $\tilde D_{i;\cK}[f]$) for $D_{i,0;\cK}[F]$ (resp. $\tilde D_{i,0;\cK}[F]$) where  $F:(\theta,\theta') \mapsto f(\theta)$.

For $i, j\in \N$, if $\cK$  is of class $\cc^{i\vee 1,j\vee 1}$, then we consider
the real-valued function defined on $\Theta^2$ by:
\begin{equation}
\label{eq:def-deriv-K}
\cK^{[i, j]}=\tD_{i,j;\cK}[\cK].
\end{equation}
In particular, when $\cK$ is of class $\cc^2$, we have:
\begin{equation}
\label{eq:K-g}
\cK^{[0, 0]}=\cK 
\quad\text{and}\quad
\cK^{[1, 1]}(\theta, \theta)=1.
\end{equation}

\subsubsection{The kernel associated to the dictionary of features}
Let $T\in  \N$ be fixed and assume that Assumption~\ref{hyp:g>0}
holds. We associate to the dictionary of features $(\varphi_T(\theta), \theta \in \Theta)$ a kernel $\cK_T$ on $\Theta^2$ defined by: 
\begin{equation}
\label{eq:def-KT}
\cK_T(\theta, \theta')=\langle \phi_{T}(\theta), \phi_{T}(\theta')
\rangle_T =\frac{\langle \varphi_T(\theta), \varphi_T(\theta')
	\rangle_T}{\norm{\varphi_T (\theta)}_T\norm{\varphi_T (\theta')}_T}\cdot
\end{equation}

In the following, for an expression $A$ we will often replace   $A_{\cK_*}$ by $A_*$ where
$*$ is  $T$ or $\infty$.

We remark that under Assumptions~\ref{hyp:reg-f} and~\ref{hyp:g>0} the definitions \eqref{def:g_T} and~\eqref{eq:def-gK} are consistent by \cite[Lemma 4.3]{butucea22}. Furthermore,  we have that the kernel  $\cK_T$ is  of
class $\cc^{3,3}$ on $\Theta^2$ and for $i, j\in \{0, \ldots, 3\}$ and for any
$\theta, \theta'\in \Theta$:
\begin{equation}
\label{def:derivatives_kernel}
\cK_T^{[i,j]}(\theta,\theta') = \langle
\tD_{i;T}[\phi_{T}](\theta), \tD_{j;T}[\phi_{T}](\theta')
\rangle_T,
\end{equation}
\begin{equation}
\label{eq:formula_K_00}
\sup_{\Theta^2} |\cK_T^{[0,0]}|\leq 1,
\quad  \cK_T^{[0,0]}(\theta,\theta)=1,
\quad   \cK_T^{[1,0]}(\theta,\theta) = 0, \quad 
\cK_T^{[2,0]}(\theta,\theta) = -1 \quad \text{and} \quad
\cK_T^{[2,1]}(\theta,\theta) = 0. 
\end{equation}

In practice, the kernel $\ck_T$ may be difficult to handle. It might be convenient to approximate $\ck_T$ by a kernel $\ck_\infty$ for which some assumptions will be easier to check. 
We give necessary conditions that an approximating kernel $\ck_\infty$ must verify. Then we define a quantity measuring the precision of the approximation of $\ck_T$ by $\ck_\infty$ over some compact set $\Theta_T \subseteq \Theta$.

Let us first define for a kernel $\cK$  of class $\cc^{3,3}$ the function on $\Theta$:
\begin{equation}
\label{eq:def-h_K}
h_\cK(\theta)=\cK^{[3,3]}(\theta, \theta).
\end{equation}
We also and simply write for a real-valued function $f$ on $\Theta$ of class $\cc^i$:
\[
f^{[i]}= \tilde D_{i;T}[f].
\]

The following assumption gathers the conditions that an approximating kernel $\ck_\infty$ must sastify.
\begin{hyp}[Necessary conditions on the asymptotic kernel $\cK_\infty$]
	\label{hyp:Theta_infini}
	The symmetric kernel $\cK_\infty $ defined on $\Theta^2$ is of 
	class $\cc^{3,3}$,   the  function
	$g_\infty $  defined by~\eqref{eq:def-gK} on $\Theta$  is positive and
	locally bounded (as well as of class $\cc^2$), and we have $\cK_\infty (\theta,\theta)=- \ck_\infty
	^{[2, 0]}(\theta, \theta)=1$ for $\theta\in \Theta$. The set  $\Theta_\infty \subseteq \Theta$ is  an
	interval and we have:
	\begin{equation}
	\label{def:M_h_M_g}
	m_g:= \inf_{\Theta_\infty } \, g_{\infty} > 0, \quad  L_3:=
	\sup_{\Theta_\infty } \, h_{\infty} < + \infty,
	\quad\text{and}\quad	L_{i,j}: = \sup_{\Theta_{\infty}^2} |\cK_\infty ^{[i, j]}|
	<+\infty
	\quad\text{for all  $ i,j \in \{0, 1, 2\}$}.
	\end{equation}
\end{hyp}
We stress that the interval $\Theta_\infty$ is possibly unbounded contrary to the set $\Theta_T$ which is compact.
\medskip 

Under assumption \ref{hyp:Theta_infini}, we derive from the kernel $\ck_\infty$ the Riemannian metric $\dI$ as in \eqref{eq:def-Riemann-dist-v2}.
One can show that  the metrics $\dT$ and $\dI$ are strongly equivalent on the compact set $\Theta_T^2$. Indeed, we have:
\begin{equation}
\label{eq:equi-dT-dI}
  \dI / \RT\leq  \dT \leq  \RT\,  \dI,
\end{equation}
where $\rho_T$ is a finite positive constant defined by:
\begin{equation}
\label{eq:def-rho}
\RT=\max \left(\sup_{\Theta_T} \sqrt{\frac{g_T}{g_\infty
}},\sup_{\Theta_T} \sqrt{\frac{g_\infty }{g_T }} \right). 
\end{equation}

We then give an assumption on the quality of approximation of $\cK_T$ by
$\cK_\infty $.
We set:
\begin{equation}
\label{def:V_1}
\DT=\max( \DT^{(1)}, \DT^{(2)})
\quad\text{with}\quad
\DT^{(1)}=\max_{i,j\in \{0, 1, 2\} }\, \sup_{\Theta_T^2} |
\cK_T^{[i,j]} - \cK_\infty ^{[i,j]}|
\quad\text{and}\quad
\DT^{(2)}=\sup_{\Theta_T} |h_T - h_\infty |. 
\end{equation}
\begin{hyp}[Quality of the approximation]
	\label{hyp:close_limit_setting}
	Let $T \in \N$ be  fixed. Assumptions~\ref{hyp:g>0} 
	and~\ref{hyp:Theta_infini} hold, the interval  $\Theta_T\subset
	\Theta_\infty $ is a compact interval,  and  we have:
	\[
	\DT\leq   L_{2,2} \wedge L_3.
	\]
\end{hyp}

%%%%%%%%%%%%%%%%%%%%%%%%
\section{Main results}
\label{sec:main_result}
%%%%%%%%%%%%%%%%%%%%%%%%

\subsection{General bound on the prediction error}
The main goal of this paper is to bound the prediction error \eqref{eq:prediction_error} associated to the estimators defined in \eqref{eq:generalized_lasso}.
We first give a bound that holds with a controled probability in the general case where the penalty of the optimization problem \eqref{eq:generalized_lasso} is the norm $\norm{\cdot}_{\ell_1,L^p(\nu)}$ with $p \in [1,2]$. The bound is expressed as a function of the tuning parameter $\kappa$, the sparsity $s$, the mass of the measure $\nu$ and the parameter of the penalty $p$. 
It stands on an event whose probability is bounded from below by tails of distributions of random variables defined by taking the supremum over the compact set $\Theta_T$ and the norm $\norm{\cdot}_{L^q(\nu)}$ of real-valued processes indexed on $ \spi \times \Theta_T$ of the form:
\[
X(\idx,\theta) = \left
\langle W_T(\idx), g(\theta) \right \rangle_T,
\] 
for some smooth functions $g : \Theta_T \rightarrow H_T$ related to the dictionary of features and  where $q$ is the conjugate of $p$ in the sense that $1/q +1/p = 1$.

The assumptions on the regularity of the dictionary, the regularity of the limit kernel and the proximity  to the limit kernel are the same as those from  \cite[Theorem 2.1]{butucea22}. Regarding the noise, we only require that it belongs almost surely to $L^q(\nu,H_T)$. We highlight that the Theorem below is proven under the existence of certificate functions. Those certificates generalize those of \cite[Theorem 2.1]{butucea22}. (In particular, they reduce to those in \cite{butucea22} when $\nu$ is a Dirac measure.) A construction of certificates has been proposed in \cite{golbabaee2020off} for the case where $\nu$ is the counting measure. Our construction is slightly different and covers the general case where $\nu$ can be any finite positive measure, see Remark D.4. of the supplementary material. We shall give in Section \ref{sec:construction_certificates} sufficient conditions for their existence. For all $\idx \in \spi$, we note $\cq^\star(\idx)$  the  finite set  of the
parameters of  the active  features appearing  in $Y(\idx)$. We assume  that the unknown
number of  active features $s$  in the observation  $Y$ is bounded  by a
constant                  $K$,                  that                  is:
\begin{equation*}
\label{eq:Q}
K  \geq  \operatorname{Card}(\bigcup_{\idx \in  \spi}  \cq^\star(\idx))
:=s.
\end{equation*}
In the following we make a slight abuse of notation by writing $\cq^\star$ instead of $\bigcup_{\idx \in  \spi}  \cq^\star(\idx)$.

It turns out that we can construct such certificates provided the elements of the set $\cq^\star$ defined above are pairwise separated with respect to a Riemannian metric. We remark that the separation does not depend on the space $(\spi, \mathcal{F},\nu)$. In particular, in the example where $\spi$ is a finite set of cardinality $n$, increasing $n$ does not improve or deteriorate the separation.

We state the main result of this paper that is proved in Section \ref{sec:proofsmaintheorem}.
\begin{theorem}
	\label{maintheorem} Let $T \in \N$. Let be $p \in [1,2]$ and $q \in [2,+\infty]$ such that $1/p+1/q=1$. When $p=1$, we assume that $\spi$ is finite.
	Assume we observe the random element $Y$ of $L_T$ under the regression model (\ref{eq:model}) with a noise $W_T$ belonging to $L^q(\nu,H_T)$   almost surely and unknown parameters  $B^\star \in L^2(\nu,\R^K)$ and $\vartheta^\star= \left (
	\theta_1^\star,\cdots,\theta_K^\star\right )$ a vector with entries in
	$ \Theta_T$ (compact interval of $\R$). Let us suppose that the following assumptions hold :
	\begin{propenum}

		\item\textbf{Regularity           of            the           dictionary
			$\varphi_T$:}\label{hyp:reg_dic_theorem}  The   dictionary  function
		$\varphi_T$     satisfies     the     smoothness     conditions 
		\ref{hyp:reg-f} .    The    function   $g_T$ satisfies the positivity condition		\ref{hyp:g>0}. 
		\item\textbf{Regularity of the limit kernel:}
		The   kernel   $\ck_{\infty}$ and the functions  $ g_{\infty}$   and
		$h_{\infty}$,   defined  on an interval $\Theta_\infty \subset
		\Theta$, satisfy  the  smoothness  conditions  of  Assumption
		\ref{hyp:Theta_infini}.
		\item\textbf{Proximity to the limit kernel:} \label{hyp:V_T_theorem} The
		kernel $\cK_T$ defined from  the dictionary is
		sufficiently close to the limit  kernel $\cK_\infty$ in the sense that
		Assumption \ref{hyp:close_limit_setting} holds.
		\item\textbf{Existence of  certificates:}\label{hyp:existence_certificate_theorem} The non-empty set of  unknown parameters
		$\cq^\star= \{\theta^\star_k,  \, k \in S^\star\}  $, with
		$S^\star = \{k, \, \norm{B_{k}^{\star}}_{L^2(\nu)} \neq 0 \}$,  satisfies
		Assumptions  \ref{assumption1} and  \ref{assumption2}  with the  same
		$r >0$.
	\end{propenum}
	Then, there exist  finite positive constants $\mathcal{C}$, $\mathcal{C}_0$  
	depending on $r$ and on the kernel $\cK_\infty$ defined on $\Theta_\infty$ such that
	we have the prediction error bound of the estimators $\hat B $ and $\hat{\vartheta}$ defined for a tuning parameter $\kappa > 0$ (in
	\eqref{eq:generalized_lasso})  given by:
	\begin{equation}
	\label{eq:main_theorem}
	\begin{aligned}
	 \frac{1}{\nu(\spi)}\norm{\hat B \Phi_{T}(\hat{\vartheta}) -
		B^{\star}\Phi_{T}(\vartheta^{\star}) }_{L_T}^2
	&\leq  \mathcal{C}_0 \,  \sparse  \, \nu(\spi)^{\frac{2}{p}}\, \kappa^2,
	\end{aligned}
	\end{equation}
	with  probability larger than 	
	\begin{equation}
	\label{eq:proba_maintheorem}
	1 - \sum_{i=0}^2 \mathbb{P} \left( \frac 1{\nu(\spi)} M_{i} >  \mathcal{C} \, \kappa \,  \right),
	\end{equation} 
	where $M_i$ is defined by:
	\begin{equation}
	\label{def:M}
	M_{i} = \underset{\theta \in \Theta_T}{\sup}\norm{\left
		\langle W_T, \phi_{T}^{[i]}(\theta) \right \rangle_T }_{{L^q(\nu)}}, \quad \text{for } i=0,1,2.
	\end{equation}	
\end{theorem} 
We show in Section \ref{sec:bounding_M} below that the random variables $M_i$ can be bounded explicitly with high probability when a finite number of signals is observed and the noise (assumed Gaussian) satisfies Assumption \ref{hyp:bruit}. Giving explicit bounds in the case where an infinite number of signals (possibly a continuum) is observed is beyond the scope of this paper and could be an avenue for future work.
\begin{remark}[On the choice of $\kappa$]
	We typically choose $\kappa$ in \eqref{eq:main_theorem} as small as possible giving a global bound on the prediction risk small, such that the event on which the bound stands occurs with a sufficiently large probability.
\end{remark}

\begin{remark}[On the dimension $K$]
	The bound $K$ on the sparsity $s$ appears neither in the upper bound on the prediction error \eqref{eq:main_theorem} nor in the lower bound on the probability \eqref{eq:proba_maintheorem}. Thus, it can be taken arbitrarily large. This was already the case in \cite{butucea22} where $\spi$ is a singleton and $\nu$ is a Dirac measure, see Remark 2.4 therein.
\end{remark}

\subsection{Explicit bounds for Gaussian noise and finite number of signals}
\label{sec:bounding_M}
It is not straightforward to establish tail bounds for the random variables $M_i$ defined in Theorem \ref{maintheorem}. However, if the noise process for fixed $z$ in $\cz$ is centered Gaussian, for the cases $p=q=2$ and $p=1$ together with $q=+\infty$, this can be done using Rice formulae (see \cite{azais2009level} for a complete overview of Rice formulae).

We will give an explicit lower bound for the probability \eqref{eq:proba_maintheorem}. The lower bound will depend on the parameter $T$ and the number of signals $n = \operatorname{Card}(\spi)$ assumed to be finite here. Thus, we will be able to give a convergence rate towards zero for the prediction error with respect to these parameters.

In order to use tail bounds for the random variables $M_i$, $i\in \{0,1,2\}$, from Theorem \ref{maintheorem}, we state additional assumptions on the noise $W_T$. 
We make the following  assumption on the noise  process $W_T$, where the decay rate  $\Delta_T>0$ controls the noise variance decay as the parameter $T$ grows and $\sigma>0$ is the
intrinsic noise level.

\begin{hyp}[Admissible noise]	\label{hyp:bruit}
	Let $T \in \N$. Assume that the set $\spi$ is finite. The processes $(W_T(\idx), \, \idx \in \spi)$ are independent copies of a noise process $w_T$. The noise process $w_T$ belongs to $H_T$ almost surely and, there exist a noise level $\sigma>0$ and a decay rate $\Delta_T>0$ such that for all $f\in H_T$ the random variable $\langle f,w_T  \rangle_T$  is a centered Gaussian random variable satisfying:
	\begin{equation}
	\label{eq:hyp:bruit}
	\Var \left( \langle f,w_T  \rangle_T \right)\leq \sigma^2 \,
	\Delta_T\,  \norm{f}_T^2.
	\end{equation}
	
\end{hyp}

\subsubsection{The case $p=2$ and $\spi$ finite}
We state a corollary of Theorem \ref{maintheorem} for the specific case where $\nu$ is an atomic measure composed of $n$ atoms and the penalty of the optimization problem \eqref{eq:generalized_lasso} is a mixed $(\ell_1,L^2(\nu))$ norm. The proof is given in Section B of the supplementary material.

We denote by $|\Theta_T|_{\mathfrak{d}_T} $ the diameter of the interval $\Theta_T$ with respect to the Riemannian metric $\mathfrak{d}_T$ associated to the kernel $\ck_T$ and defined in \eqref{eq:def-Riemann-dist-v2}.

\begin{corollary}
	\label{cor:chi2}
	Let $T \in \N$. We fix $p=q=2$.  We assume that $ \operatorname{Card}(\spi) = n < + \infty$ and that the measure $\nu$ is $\nu = \sum_{\idx \in \spi} a_\idx \delta_{z}$ where $\delta_\idx$ denotes a Dirac measure located in $\idx \in \spi$ and $(a_\idx, \idx \in \spi)$ are non-negative real numbers. Assume we observe the random element $Y$ of $L_T$ under the regression model (\ref{eq:model}) with unknown parameters  $B^\star$ in $L^2(\nu, \R^K)$ (which can be identified with $\R^{n\times K}$) and $\vartheta^\star= \left (
	\theta_1^\star,\cdots,\theta_K^\star\right )$ a vector with entries in
	$ \Theta_T$, a compact interval of $\R$,   such that Points \ref{hyp:reg_dic_theorem}-\ref{hyp:existence_certificate_theorem} of Theorem \ref{maintheorem} are satisfied and the noise process $W_T$ satisfies Assumption \ref{hyp:bruit} for a noise level $\sigma>0$ and a decay rate for the noise variance $\Delta_T>0$.
	
	Then, there exist finite positive constants $\mathcal{C}_0$, 
	$\mathcal{C}_1$, $\mathcal{C}_2$, 
	depending on the kernel $\cK_\infty$ defined on $\Theta_\infty$ and on $r$ such that for 
	any $\tau > 1$ and a tuning parameter:
	$$
	\kappa \geq \mathcal{C}_1\sigma\sqrt{\frac{\norm{a}_{\ell_\infty}  \Delta_T \, n }{\nu(\spi)^2}} \left (1 + \sqrt{1 + \frac{\log(\tau)}{n}} \right ),
	$$
	where $\norm{a}_{\ell_\infty} = \max_{\idx \in \spi} |a_\idx|$, we have the following prediction error bound of the estimators $\hat B $ and $\hat{\vartheta}$ defined in
	\eqref{eq:generalized_lasso}:
	\begin{equation}
	\label{eq:cor:chi2}
	\begin{aligned}
	\frac{1}{\nu(\spi)}\norm{\hat B \Phi_{T}(\hat{\vartheta}) -
		B^{\star}\Phi_{T}(\vartheta^{\star}) }_{L_T}^2
	&\leq  \mathcal{C}_0 \,  \sparse \, \nu(\spi) \, \kappa^2,
	\end{aligned}
	\end{equation}
	with  probability larger than $1- \mathcal{C}_2 \left ( \frac{1}{\tau} + \frac{ |\Theta_T|_{\mathfrak{d}_T }F(n)  }{\sqrt{\tau}} \right )$ with  a sequence $F(n) \asymp \sqrt{n}\expp{-n/2}$.
\end{corollary}

\begin{rem}[Comparison to the Group-Lasso estimator]
	\label{rem:cor1_comp}
	Assume that the Hilbert space $H_T = \R^T$ is endowed
	with  the Euclidean scalar product and Euclidean norm $\norm{\cdot}_{\ell_2}$.  Let $\spi = \{1,\cdots,n\}$ and let $\nu$ be the counting measure on $\spi$, \emph{i.e.} $\nu = \sum_{k=1}^n \delta_k$. Notice that in this setting $L_T = L^2(\nu,H_T)$ is of finite dimension and can be identified with $\R^{n \times T}$. Assume that the observation $Y \in L_T$ comes from the model \eqref{eq:model} where for any $i \in \{1,\cdots,n\}$, $W_T(i)$ is a Gaussian vector in $\R^T$ with independent entries of variance $\sigma^2$. Assume also that the Gaussian vectors $(W_T(i), \, 1 \leq i \leq n)$ are independent. Thus,  Assumption~\ref{hyp:bruit}  holds  with  an
	equality in \eqref{eq:hyp:bruit} and:
	\[
	\Delta_T=1. 
	\]
	
	We first consider  that the parameters $\vartheta^\star$  are known.  In
	this case,  the model becomes the  classical high-dimensional multiple linear regression
	model  and the  Group-Lasso estimator $\hat{B}_{L}$ can be  used to  estimate $B^\star$
	under     coherence    assumptions     on     the    finite dictionary made of the rows of the matrix $\Phi^\star = \Phi_T(\vartheta^\star) \in \R^{K\times T}$    (see
	\cite{bickel2009simultaneous}). The authors of \cite{lounici2011oracle} showed that the prediction error associated to the Group-Lasso estimator satisfies the bound:
	\begin{equation}
	\label{eq:mse}
	\frac 1{n \, T}\sum_{i=1}^n \|(\hat B_{L}(i) - B^\star(i))\Phi^\star\|_{\ell_2}^2  \lesssim\frac{\sigma^2 \, \sparse}{T} \left ( 1 + \frac{\log(K)}{n} \right ), 
	\end{equation}
with	high   probability, larger than $1 - 1/K^\gamma$ for some positive constant $\gamma > 0$ (note that the roles of $T$ and $n$ are reversed in their paper).   Furthermore,   in    the  case   where
	$B^{\star}$     is    an     unknown    $\sparse$-sparse     mapping,
	$\vartheta^{\star}$ is known  and $\Phi^\star$
	verifies a  coherence property, then  lower  bounds of
	order $  {\sigma^2\, \sparse (1+ \log (K/\sparse)}/{n})/T$ in expected
	value can be established.
	The  non-asymptotic prediction lower bounds
	for the prediction error given in \cite{lounici2011oracle} are for $2s < K$:
	\[
	\inf_{\hat B}\,  \sup_{B^\star \, s- \text{sparse} }
	\E\left[\frac 1{n \, T}\sum_{i=1}^n \|(\hat B(i) - B^\star(i))\Phi^\star\|_{\ell_2}^2\right]\geq C\cdot
	\frac{  \sigma^2\, \sparse}{T} \left (1 + \frac{\log(K/s)}{n} \right) ,
	\]
	where  the infinimum  is  taken over  all  the estimators  $\hat B$
	(measurable  functions of the obervation  $Y$ taking their values in $L^2(\nu,\R^K)$) and for
	some  constant  $C>0$  free  of   $s$, $K$, $n$  and  $T$. 
	\medskip
	
	When  the linear coefficients $B^\star$ and the  parameters
	$\vartheta^\star$ are unknown, Corollary~\ref{cor:chi2} gives an upper bound for  the prediction risk which  is similar to that of the linear case.   Consider    the   estimators    from
	\eqref{eq:generalized_lasso} with $p=2$. Assume that  the Riemannian  diameter of  the set
	$\Theta_T$ is bounded by a constant  free of $T$. By  dividing
	\eqref{eq:cor:chi2}  by $T$,  we obtain  from
	Corollary                     \ref{cor:chi2}                     with:
	\[
	\kappa  =  \mathcal{C}_1\sigma\sqrt{ \frac{1 }n } \left (1 + \sqrt{1 + \frac{\log(\tau)}{n}} \right ) \quad  \text{and}\quad
	\tau  = T^\gamma \quad \text{for  some given } \gamma>0,
	\]
	that  with high  probability,
	larger than $1- C'/T^\gamma - C''F(n)/T^{\gamma/2} $:
	\begin{equation}
	\label{eq:mse2}
	\frac 1{n \, T}\sum_{i=1}^n \norm{\hat B(i)\Phi_{T}(\hat{\vartheta}) -
		B^{\star}(i)\Phi_{T}(\vartheta^{\star}) }_{\ell_2}^2 \lesssim
	\frac{\sigma^2 \, \sparse}{T} \left ( 1 + \frac{\log(T)}{n} \right ) .
	\end{equation}
	We identify two regimes depending on the ratio $\log(T)/n$. Indeed, when $\log(T)/n \gg 1$ the bound \eqref{eq:mse2}   behaves as $\frac{\sigma^2 \, \sparse \log(T)}{n \, T }$ and stands with probability that converges towards $1$ at the rate $F(n)/T^{\gamma/2}$. 	On the contrary,  when $\log(T)/n \ll 1$ the bound \eqref{eq:mse2}   is  of order $\frac{\sigma^2 \, \sparse}{ T }$ and stands with probality that converges towards $1$ at the rate $1/T^{\gamma}$.
\end{rem}
\subsubsection{The case $p=1$ and $\spi$ finite}
We apply Theorem \ref{maintheorem} to the particular case $p=1$. It turns out that for $q=+\infty$, tail bounds for the random variables $M_j$ with $j=0,1,2$ can be established from Rice formulae for smooth Gaussian processes. The following Corollary is proved in Section C of the supplementary material.
\begin{corollary}
	\label{cor:extension}
	Let $T\in \N$. We fix $p=1, q=+\infty$.  We assume that $ \operatorname{Card}(\spi) = n < + \infty$ and that the measure $\nu$ is $\nu = \sum_{\idx \in \spi} a_\idx \delta_{z}$ where $\delta_\idx$ denotes a Dirac measure located in $\idx \in \spi$ and $(a_\idx, \idx \in \spi)$ are non-negative real numbers. Assume we observe the random element $Y$ of $L_T$ under the regression model (\ref{eq:model}) with unknown parameters  $B^\star$ in $ L^2(\nu, \R^K)$ (which can be identified with $\R^{n\times K}$) and $\vartheta^\star= \left (
	\theta_1^\star,\cdots,\theta_K^\star\right )$ a vector with entries in
	$ \Theta_T$, a compact interval of $\R$,   such that Points \ref{hyp:reg_dic_theorem}-\ref{hyp:existence_certificate_theorem} of Theorem \ref{maintheorem} are satisfied and the noise process $W_T$ satisfies Assumption \ref{hyp:bruit} for a noise level $\sigma>0$ and a decay rate for the noise variance $\Delta_T>0$.
	
	Then, there exist finite positive constants $\mathcal{C}_0$, 
	$\mathcal{C}_3$, $\mathcal{C}_4$, 
	depending on the kernel $\cK_\infty$ defined on $\Theta_\infty$ and on $r$ such that for 
	any $\tau > 1$ and a tuning parameter:
	$$
	\kappa \geq \mathcal{C}_3\sigma\sqrt{\Delta_T \log(\tau) } / \nu(\spi),
	$$
	we have the following prediction error bound of the estimators $\hat B $ and $\hat{\vartheta}$ defined in
	\eqref{eq:generalized_lasso}:
	\begin{equation}
	\label{eq:cor:extension}
	\begin{aligned}
	\frac{1}{\nu(\spi)}\norm{\hat B \Phi_{T}(\hat{\vartheta}) -
		B^{\star}\Phi_{T}(\vartheta^{\star}) }_{L_T}^2
	&\leq  \mathcal{C}_0 \,  \sparse \, \nu(\spi)^2 \, \kappa^2,
	\end{aligned}
	\end{equation}
	with  probability larger than $1-   \mathcal{C}_4 \, n \,\left (  \frac{|\Theta_T|_{\mathfrak{d}_T}}{\tau
		\sqrt{\log \tau } }\vee \frac{1}{\tau}\right ) $.
\end{corollary}
\begin{remark}
	When the measure $\nu$ is composed of one atom, that is $n=1$. This result covers that of \cite[Theorem 2.1]{butucea22}.
\end{remark}

\begin{remark}[Comparison to other estimators] Let us set $H_T= \R^T$, $\mathcal{Z} = \{1, \cdots,n\}$, $\nu$ the counting measure and $W_T$ as in Remark \ref{rem:cor1_comp} and assume that the Riemannian  diameter of  the set $\Theta_T$ is bounded by a constant  free of $T$. We recall that in this case $\Delta_T=1$.
	By considering the estimators built from  the optimization problem \eqref{eq:generalized_lasso} with $p=1$ and applying Corollary \ref{cor:extension}, we get with:
	\[
	\kappa  =  \mathcal{C}_3\sigma\sqrt{  \Delta_T \, \log \tau} / n \quad   \text{ and } \tau  = T^{\gamma/2} \quad  \text{for  some given} \quad \gamma>1,
	\]
	that,  with probability,
	larger than $1- C\, n/T^{\gamma/2}$:
	\begin{equation}
	\label{eq:mse3}
	\frac 1{n \, T}\sum_{i=1}^n \norm{\hat B(i)\Phi_{T}(\hat{\vartheta}) -
		B^{\star}(i)\Phi_{T}(\vartheta^{\star}) }_{\ell_2}^2  \lesssim \frac{\sigma^2 \, \sparse \, \log(T)}{T} \cdot
	\end{equation}
	We note that this simultaneous estimation procedure gives the same predictors as estimating separately $n$ signals according to \cite{butucea22}, provided that the design matrix has size $K$ large enough. Separate estimation and aggregation of the bounds give the following bound $\sigma^2 \bar s \log(T)/T$ instead of \eqref{eq:mse3}, where $\bar s$ is the average sparsity of the $n$ signals. The latter bound is smaller, but is of the order $\sigma^2 s \log(T)/T$ when all signals share most of the non-linear parameters.
\end{remark}

\begin{remark}[Comparison for $p=2$ and $p=1$] 
	In Remark \ref{rem:cor1_comp}, we showed that by taking $p=2$ in the optimization problem \eqref{eq:generalized_lasso} defining the estimators $\hat B$ and $\hat \vartheta$, we obtain the bound \eqref{eq:mse2} for a well chosen  tuning parameter $\kappa$.
	When $n$ and $T$ are sufficiently large, we remark that the bound \eqref{eq:mse3} (obtained when $p=1$) is larger than the bound \eqref{eq:mse2} (obtained when $p=2$) established for the estimators from Corollary \ref{cor:chi2} and stands with a smaller probability.

Furthermore, separate estimation of each signal as in \cite{butucea22} and aggregation of the bounds give the following bound $\sigma^2 \bar s \log(T)/T$ instead of \eqref{eq:mse2} (obtained with $p=2$), where $\bar s$ is the average sparsity of the $n$ signals. Provided $\log(T)/n$ is large, this bound is always larger than \eqref{eq:mse2}, but the bounds are of the same order when all signals have disjoint sets of non-linear parameters. 

In conclusion, the Group-Nonlinear-Lasso for $p=2$  provides faster prediction rates than for $p=1$ when all signals share most of the non-linear parameters. 
\end{remark}

%%%%%%%%%%%%%%%%%%%%%%%%%%%

\section{Certificates}
\label{sec:certificates}
%%%%%%%%%%%%%%%%%%%%%%

We present the certificate functions whose existence is required  in Theorem \ref{maintheorem}. Such functions were introduced for exact  reconstruction of signals, see \cite{candes2011probabilistic}, \cite{candes2014towards}, \cite{duval2015exact}.  Exact recovery results for the simultaneous reconstruction of signals via the Group-Nonlinear-Lasso were proved in \cite{golbabaee2020off} using an extension of the certificates from \cite{duval2015exact}.  In \cite{poon2018geometry}, sufficient conditions for the existence of certificate functions were proved for a wide variety of dictionaries. The authors showed that certificates can be built provided the parameters of the features to be retrieved are well separated with respect to a Riemannian metric. This result requires some assumptions on the kernel associated to the dictionary. In particular, the kernel must be  local concave on its diagonal, strictly inferior to $1$ outside the diagonal and smooth. Their construction was used in \cite{butucea22} to establish prediction error bounds under similar assumptions on the dictionary but for a one-dimensional parameter space $\Theta$.

In this paper, we  extend the notion of certificates for our context of multiple reconstructions of signals,  following the work of \cite{golbabaee2020off}. Let us emphasize that we use a different contruction than \cite{golbabaee2020off}, see Remark D.4 of the supplementary material.

\subsection{Assumptions on the certificates} \label{sec:assump_certificates} In this section, we introduce the assumptions on the certificates. We will give later in Section \ref{sec:construction_certificates} an explicit construction and sufficient conditions for these assumptions to hold.

Let $T\in  \N$. We denote the closed ball centered at $\theta\in
\Theta_T$ with radius $r$ by:
\[
\mathcal{B}_T(\theta,r) = \left\{ \theta' \in \Theta_T,\,
\mathfrak{d}_T(\theta, \theta') \leq r\right\}  \subseteq \Theta_T.
\]
Let $r>0$ and let $\cq^{\star}$ be a subset of $\Theta_T$ containing $\sparse$ values. We call  near region of $\cq^\star$ the union of balls $\bigcup\limits_{\theta^\star\in \cq^{\star}} \mathcal{B}_T(\theta^{\star},r)$ and far region the set $\Theta_T$ minus the near region: $\Theta_T \setminus \bigcup\limits_{\theta^\star\in \cq^{\star}} \mathcal{B}_T(\theta^{\star},r)$.

\begin{hyp}[Interpolating certificate]
	\label{assumption1}
	Let $p,q \in [1,+\infty]$ such that $ p \leq q$ and $1/p +1/q = 1$, let $T \in \N$, $s \in \N^*$, $r>0$ and $\cq^{\star}$ be a subset of $\Theta_T$ containing $\sparse$ values. Suppose Assumptions \ref{hyp:reg-f} and \ref{hyp:g>0} on the dictionary $(\varphi_T(\theta), \, \theta \in \Theta)$
	and Assumption \ref{hyp:Theta_infini}  on $\ck_\infty $
	hold.  Suppose that $
	\mathfrak{d}_T(\theta,\theta') > 2r$ for all $\theta, \theta' \in
	\cq^\star \subset \Theta_{T}$. There exist finite positive constants  $C_{N}  , C_{N}', C_{F} $,
	$C_{B}$ with $C_F < 1$, depending on $r$
	and $\cK_\infty$, such that for  any measurable mapping $V: \spi \times \cq^\star \rightarrow \R$ such that for any $\theta^\star \in \cq^\star$, $\norm{V(\cdot,\theta^\star)}_{L^q(\nu)}= 1$, there exists  an element $P \in L^q(\nu, H_T)$ satisfying:
	\begin{propenum} 
		\item\label{it:as1-<1}
		For all $\theta^\star \in \cq^{\star}$ and $\theta\in
		\mathcal{B}_T(\theta^{\star},r)$, we have
		$ \norm{ \langle \phi_{T}(\theta),P \rangle_T}_{L^q(\nu)} \leq 1 -
		C_{N} \, \mathfrak{d}_T(\theta^{\star},\theta)^{2}$.
		\item\label{it:as1-ordre=2}
		For all $\theta^\star \in \cq^{\star}$ and $\theta\in
		\mathcal{B}_T(\theta^{\star},r)$, we have
		$ \norm{ \langle\phi_{T}(\theta),P  \rangle_T - V(\cdot,\theta^\star)}_{L^q(\nu)} \leq
		C_{N}' \, \mathfrak{d}_T(\theta^{\star},\theta)^{2}$.
		\item\label{it:as1-<1-c}  For all $\theta$ in $\Theta_T$, $\theta  \notin \bigcup\limits_{\theta^\star
			\in \cq^{\star}} \mathcal{B}_T(\theta^{\star},r)$ (far region), we
		have $\norm{ \langle\phi_{T}(\theta), P \rangle_T}_{L^q(\nu)} \leq  1 - C_{F}$.
		\item \label{it:norm<c} We have $\norm{P}_{L_T} \leq  C_{B} \, \sqrt{s} \, \nu(\spi)^{\frac{1}{2p}-\frac{1}{2q}}$.
	\end{propenum}
\end{hyp}
We call ``interpolating certificate" the real-valued functions $(\idx,\theta) \mapsto \left \langle \phi_T(\theta), P(\idx)\right \rangle_T$ defined on ${\spi \times \Theta}$ where $P$ is an element of $L^q(\nu, H_T)$ satisfying Points $\ref{it:as1-<1}-\ref{it:norm<c}$ from \ref{assumption1}.

We emphazise the interpolating properties of those certificates by noticing that for any $\theta^\star \in \cq^\star$ we have from Point $\ref{it:as1-ordre=2}$ for $\nu$-almost every $\idx \in \spi$ that:
\[
\left \langle \phi_T(\theta^\star), P(\idx)\right \rangle_T = V(\idx,\theta^\star).
\] 

In order to establish prediction error bounds another type of certificate functions having different interpolating properties will be needed, see \cite{candes2013super}, \cite{tang2014near}, \cite{boyer2017adapting} in this direction.
\begin{hyp}[Interpolating derivative certificate]
	\label{assumption2}
	Let $p,q \in [1,+\infty]$ such that $ p \leq q$ and $1/p +1/q = 1$, let $T \in \N$, $s \in \N^*$, $r>0$ and $\cq^{\star}$ be a subset of $\Theta_T$ containing	$\sparse$ values.  Suppose  Assumption \ref{hyp:reg-f} and \ref{hyp:g>0} on the dictionary $(\varphi_T(\theta), \, \theta \in \Theta)$
	and Assumption \ref{hyp:Theta_infini}  on $\ck_\infty $
	hold.  Suppose that $
	\mathfrak{d}_T(\theta,\theta') > 2r$ for all $\theta, \theta' \in \cq^\star \subset \Theta_{T}$. There exist finite positive constants $c_{N}, c_{F} $, $c_{B}$  depending on $r$
	and $\cK_\infty$  such that for  any measurable mapping $V: \spi \times \cq^\star \rightarrow \R$ such that for any $\theta^\star \in \cq^\star$, $\norm{V(\cdot, \theta^\star)}_{L^q(\nu)}= 1$, there exists  an element $Q \in L^q(\nu, H_T)$ satisfying:
	\begin{propenum}
		\item\label{it:as2-ordre=2}
		For all $\theta^\star \in \cq^{\star}$ and $\theta\in
		\mathcal{B}_T(\theta^{\star},r)$, we have:
		\[
		\norm{ \langle\phi_{T}(\theta),Q \rangle_T - V(\cdot,\theta^\star)\, \operatorname{sign}(\theta-\theta^{\star}) \, \mathfrak{d}_T(\theta,\theta^{\star})}_{L^q(\nu)} \leq
		c_{N}\, \mathfrak{d}_T(\theta^{\star},\theta)^{2}.
		\]	
		\item\label{it:as2-<1-c}  For all $\theta$ in $\Theta_T$ and  $\theta  \notin \bigcup\limits_{\theta^\star
			\in \cq^{\star}} \mathcal{B}_T(\theta^{\star},r)$ (far region), we
		have $\norm{\langle\phi_{T}(\theta),Q\rangle_T}_{L^q(\nu)} \leq c_{F}$.
		\item \label{it:as2-<c} We have $||Q||_{L_T} \leq  c_{B} \, \sqrt{s}\, \nu(\spi)^{\frac{1}{2p}-\frac{1}{2q}}$.
	\end{propenum}
\end{hyp}
We call ``interpolating  derivative certificate" the real-valued functions defined on $\spi \times \Theta$ by $(\idx,\theta) \mapsto \left \langle \phi_T(\theta), Q(\idx)\right \rangle_T$ where $Q$ is an element of $L^q(\nu, H_T)$ satisfying Points $\ref{it:as2-ordre=2}-\ref{it:as2-<c}$ from \ref{assumption2}.

We remark that for any $\theta^\star \in \cq^\star$ we deduce from Point $\ref{it:as2-ordre=2}$ for $\nu$-almost every $\idx \in \spi$:
\[
\left \langle \phi_T(\theta^\star), Q(\idx)\right \rangle_T = 0.
\]

\medskip

Let us remark that when $\nu$ is a Dirac measure, the norm $\norm{\cdot}_{L^q(\nu)}$ reduces to an absolute value and Assumptions \ref{assumption1} and \ref{assumption2} correspond to Assumptions 6.1 and 6.2 of \cite{butucea22}.

\medskip  

In the following, we shall often write by a slight abuse of notation $f(\theta)$ for $f(\cdot,\theta)$ when considering a function $f$ from $\spi \times \Theta$ to $\R$.

%%%%%%%%%%%%%%%%%%%%

\subsection{Construction of the certificates}
\label{sec:construction_certificates}
We give in this section sufficient conditions for Assumptions \ref{assumption1} and \ref{assumption2} to hold. These assumptions rely on the existence of real-valued functions defined on $\spi \times \Theta$ called   certificates and of the form:
\[
(\idx,\theta) \mapsto \left \langle \phi_T(\theta), P(\idx)\right \rangle_T, 
\]
where $P$ is an element of $L^q(\nu, H_T)$ satisfying some properties.

We shall follow the construction  from \cite[Theorem 2]{poon2018geometry} for interpolating certificates and generalize the contruction of \cite[Lemma 2.7]{candes2013super} for interpolating derivative certificates. In \cite[Lemma 2.7]{candes2013super}, the authors consider  certificates that are trigonometric polynomials whereas we are interested here in a more general framework. Furthermore, we remark that  the constructions aforementioned only cover the case where  $\nu$ is a Dirac measure whereas $\nu$ can be any finite positive measure in our framework. 

Once built, we will then show that our  certificates satisfy the properties required in Assumptions \ref{assumption1} and \ref{assumption2}. The proofs of the results of this section will generalize the proofs of \cite[Propositions 7.4 and 7.5]{butucea22} in order to cover the case where $\nu$ is a finite measure instead of a Dirac measure (\emph{i.e.} only one signal).

\medskip

We consider bounded kernels locally concave on the
diagonal. We shall also require the kernels to be strictly less than $1$ outside their diagonal. In order to state these properties clearly, we define for $T\in \bar \N=\N\cup\{\infty\}$ and $r>0$:

\begin{align}
\label{eq:def-e0}
\varepsilon_{T}(r)
& = 1 - \sup \left\{ |\cK_{T}(\theta,\theta')|;
\quad \theta,\theta'\in \Theta_T  \text{ such that }
\mathfrak{d}_{T}(\theta',\theta) \geq r
\right\},\\
\label{eq:def-e2}
\nu_{T}(r)
&= -\sup \left\{
\cK_{T}^{[0,2]}(\theta,\theta'); \quad 
\theta,\theta'\in \Theta_T  \text{ such that }
\mathfrak{d}_{T}(\theta',\theta) \leq r \right\}.
\end{align}
The quantities $\varepsilon_{T}(r)$ and $\nu_{T}(r)$ defined from the considered kernel $\ck_T$ and the set $\Theta_T$ will have to be positive for some $r>0$. The positivity may be difficult to show when $T\in \N$. In order to show the positivity of $\varepsilon_{T}(r)$ and $\nu_{T}(r)$, one can rather show the positivity of $\varepsilon_{\infty}(r)$ and $\nu_{\infty}(r)$ derived from an approximating kernel easier to handle and use \cite[Lemma 7.1]{butucea22}.

We define the set $\Theta_{T , \delta}^s\subset \Theta_T ^s$
of  vector of parameters of dimension $s\in
\N^*$ and  separation $\delta>0$ as:
\begin{equation}
\label{eq:def-set}
\Theta^s_{T , \delta}= \Big  \{  (\theta_1,\cdots,\theta_s) \in \Theta_T ^s\,
\colon\,  \mathfrak{d}_{T}(\theta_{\ell},\theta_{k}) >  \delta \text{  for
	all distinct } k, \ell\in \{1, \ldots, s\}   \Big \}.
\end{equation}

Let us define for $i,j=0, 1, 2$ (assuming the kernel
$\cK_T$ is smooth enough) and  $ \vartheta = (\theta_1, \ldots,
\theta_\sparse)\in \Theta_{T}^\sparse$ the $s\times s$ matrix:
\begin{equation}
\label{def:gamma-ii}
\cK_T^{[i,j]}(\vartheta) = \left
(\cK_T^{[i,j]}(\theta_{k},\theta_{\ell}) \right
)_{1\leq k,\ell \leq s}. 
\end{equation}
Let $I$ be the  identity matrix of size $s\times s$.

Using  the  convention  $\inf  \emptyset=+\infty $,
We define:
\begin{equation}
\label{eq:def-delta-rs2}
\delta_T(u,s) =    \inf \Big\{ \delta>0\, \colon\, A_{T, \ell_\infty }(\vartheta)  \leq u,
\vartheta  \in \Theta^s_{T , \delta} \Big\}, 
\end{equation}
where:
\begin{multline}
\label{eq:def-MT}
A_{T, \ell_\infty }(\vartheta)=
\max\left(
\norm{I - \cK_T^{[0,0]}(\vartheta) }_{\op,\ell_\infty},
\norm{I - \cK_T^{[1,1]}(\vartheta) }_{\op,\ell_\infty},
\norm{I+\cK_T^{[2,0]}(\vartheta) }_{\op,\ell_\infty},\right.\\ 
\left. 
\norm{\cK_T^{[1,0]}(\vartheta) }_{\op,\ell_\infty} , \norm{\cK_T^{[0,1]}(\vartheta) }_{\op,\ell_\infty},
\norm{\cK_T^{[1,2]}(\vartheta) }_{\op,\ell_\infty}  
\right),
\end{multline}
and $\norm{\cdot}_{\op,\ell_\infty}$ denotes the operator norm associated to the sup-norm $\norm{\cdot}_{\ell_\infty}$, that is for a matrix $A \in \R^{s\times s}$, 
\[
\norm{A}_{\op,\ell_\infty} = \sup_{x \in \R^s, \norm{x}_{\ell_\infty} \leq 1} \norm{Ax}_{\ell_\infty}.
\]

We    define     quantities    which    depend    on  $\cK_\infty$, $\Theta_\infty$ and on real parameters $r>0$
and $\rho \geq 1$:
\begin{equation}
\label{eq:def_H}
\begin{aligned}
H_{\infty}^{(1)}(r,\rho)
&= \inv{2} \wedge L_{2,0} \wedge L_{2,1}  \wedge
\frac{\nu_{\infty}(\rho r)}{10} \wedge
\frac{\varepsilon_{\infty}(r/\rho) }{10},\\ 
H_{\infty}^{(2)} (r,\rho)
&= \frac{1}{6} \wedge \frac{8\,\varepsilon_{\infty} (r/\rho) }{10(5 + 2
	L_{1,0})}  \wedge 
\frac{8\,\nu_{\infty}(\rho r)}{9 ( 2L_{2,0} + 2 L_{2,1} + 4)} ,
\end{aligned}
\end{equation} 
where  the constants $L_{i,j}$ are defined  in
\eqref{def:M_h_M_g}.    

We give sufficient conditions for Assumption \ref{assumption1} to hold. The proof of the following result is given in Section D.1 of the supplementary material.
\begin{prop}[Interpolating certificate]
	\label{prop:certificat_interpolating}
	Let  $T \in  \N$,  $s  \in  \N^{*}$, $\rho \geq 1$, $r>0$ and $p,q \in [1,+\infty]$ such that $ p \leq q$ and $1/p +1/q = 1$. We  assume that: 
	\begin{propenum}
		\item \textbf{Regularity of the dictionary
			$\varphi_T$:} \label{hyp:theorem_certificate_regularity}
		Assumptions \ref{hyp:reg-f} and \ref{hyp:g>0} hold.

		\item \label{hyp:theorem_certificate_concavity}
		\textbf{Regularity of the limit kernel $\ck_\infty$:}
		Assumption~\ref{hyp:Theta_infini} holds, we have $r \in \left (0,1/\sqrt{2L_{2,0}} \right )$, and also $\varepsilon_{\infty}(r/\rho)  > 0$ and $\nu_{\infty}(\rho r)   > 0$.
		\item \label{hyp:theorem_certificate_separation}{\bf Separation of the non-linear parameters:}
		There exists $u_{\infty} \in \left
		(0,H_{\infty}^{(2)}(r,\rho) \right ) $ such that:
		\[
		\delta_{\infty}(u_{\infty},s) < + \infty.
		\]
		\item \label{hyp:theorem_certificate_metric}
		\textbf{Closeness of the
			metrics $\mathfrak{d}_T$ and  $\mathfrak{d}_\infty$:}
		We have $\rho_T \leq \rho$.
		\item \label{hyp:theorem_certificate_approximation}
		\textbf{Proximity of
			the kernels $\cK_T$ and $\cK_\infty $:} 
		\[
		\mathcal{V}_T \leq H_{\infty}^{(1)}(r,\rho)
		\quad\text{and}\quad  (s-1) \mathcal{V}_T \leq H_{\infty}^{(2)}(r,\rho)-u_{\infty}.
		\]
	\end{propenum}
	Then,  with the positive constants:
	\begin{equation}
	\label{eq:cstes-certif1}
	C_N = \frac{\nu_{\infty}(\rho r)}{180},
	\quad
	C_N'=\frac{5}{8}L_{2,0} + \frac{1}{8}L_{2,1} + \frac{1}{2},
	\quad
	C_B = 2
	\quad\text{and}\quad
	C_F = \frac{\varepsilon_{\infty}(r/\rho)}{10} \leq 1, 
	\end{equation}
	Assumption~\ref{assumption1} holds (with the same $r$) for any subset
	$\cq^{\star}=\{\theta^\star _i,  \, 1\leq i\leq s\}$ such that for all
	$\theta\neq \theta'\in \cq^{\star}$: 
	\[
	\dT(\theta, \theta')> 2 \, \max(r, \rho_T   \,
	\delta_{\infty}(u_{\infty},s))  .
	\]
\end{prop}

We state a second result that gives sufficient conditions for Assumption \ref{assumption2} to hold. The proof is given in Section D.2 of the supplementary material.
\begin{prop}[Interpolating derivative certificate]
	\label{prop:certificat2}
	Let $T \in  \N$, $s \in \N^{*}$ and $p,q \in [1,+\infty]$ such that $ p \leq q$ and $1/p +1/q = 1$.  We assume  that:
	\begin{propenum}
		\item   \label{hyp:theorem_certificate_2_regularity} \textbf{Regularity
			of the dictionary $\varphi_T$:}
		Assumptions~\ref{hyp:reg-f} and~\ref {hyp:g>0}  hold.

		\item \label{hyp:theorem_certificate_2_regularity_kernel}
		\textbf{Regularity of the limit kernel $\ck_\infty$:}
		Assumption~\ref{hyp:Theta_infini}  holds.
		
		\item  \label{hyp:theorem_certificate_2_separation} \textbf{Separation of the non-linear parameters:} There exists $u'_{\infty} \in (0, 1/6)$,
		such that:
		\[
		\delta_{\infty}(u'_{\infty},s) < + \infty. 
		\]
		\item  \label{hyp:theorem_certificate_2_approximation}
		\textbf{Proximity of
			the kernels $\cK_T$ and $\cK_\infty $:} 
		We have:
		\[
		\mathcal{V}_T \leq 1
		\quad\text{and}\quad 
		(s-1) \mathcal{V}_T + u' _\infty \leq 1/6.
		\]
	\end{propenum}
	Then,  with the positive constants:
	\begin{equation}
	\label{eq:cstes-certif2}
	c_N = \frac{1}{8}L_{2,0} + \frac{5}{8}L_{2,1}  + \frac{7}{8},
	\quad
	c_B=2 
	\quad\text{and}\quad
	c_F  = \frac{5}{4}L_{1,0} +\frac{7}{4} ,
	\end{equation}
	Assumption~\ref{assumption2} holds for  any $r>0$ and   any subset
	$\cq^{\star}=\{\theta^\star _i,  \, 1\leq i\leq s\}$ such that for all
	$\theta\neq \theta'\in \cq^{\star}$:
	\[
	\dT(\theta, \theta')> 2 \, \max(r, \rho_T   \,
	\delta_{\infty}(u'_{\infty},s))  .
	\]
\end{prop}

The assumptions of Proposition \ref{prop:certificat_interpolating} (resp.  \ref{prop:certificat2}) are identical to those of \cite[Proposition 7.4 ]{butucea22} (resp. \cite[Proposition 7.5]{butucea22}). It is not surprising since those results are based on the same construction of certificates. In order to build a certificate $\eta : (\idx ,\theta) \mapsto \R$ satisfying Assumption \ref{assumption1} or \ref{assumption2}, we shall build for every element $\idx \in \spi$ certificate functions $\eta_\idx(\theta) \mapsto \R$ following the  same construction as in \cite{butucea22} and set $\eta(\idx,\theta) = \eta_\idx(\theta)$.
The functions $\eta_\idx$ will be coupled through interpolated values on $\cq^\star$.

%%%%%%%%%%%%%%%%%%%%%%%%%%%%%%%%%%%%%%%%%%%%%%%%%
\section{Proof of Theorem \ref{maintheorem}}
\label{sec:proofsmaintheorem}
%%%%%%%%%%%%%%%%%%%%%%%%%%%%%%%%%%%%%%%%%%%%%%%%%%%%
In this section, we sketch the proof of Theorem \ref{maintheorem}.
We extend the proof of \cite[Theorem 2.1]{butucea22} to the case of a finite measure $\nu$ that is not necessarily a Dirac measure. When compared to the (now) standard proofs for Group-Lasso, this proof has the major difficulty that the design matrix is not observed but parametrized by an unknown sparse large vector $\vartheta^\star$. A Taylor expansion of second order using the metric induced by the statistical model at hand is used. Moreover, the coherence assumptions are here replaced by the existence of interpolating functions, the so called certificates.

We decompose the risk over values of estimated non-linear parameters $\hat \theta_\ell$ which are in a neighborhood of the true values $\theta^\star_k$ and those which are far away. Linear functionals of the noise depending on some $\theta \in \Theta_T$ appear in the bounds and we use probabilistic tail bounds on the suprema of these functionals over all possible values of $\theta$. 
Let us bound the squared prediction error: $$\hat{R}_T^2 := \frac{1}{{\nu(\spi)}}\norm{\hat{B}\Phi_{T}(\hat{\vartheta}) - B^{\star}\Phi_{T}(\vartheta^{\star}) }^2_{L_T}.$$  The predicition error corresponds to the integration on $\spi$ of the prediction error for one signal.

By definition \eqref{eq:generalized_lasso} of $\hat{B}$ and
$\hat{\vartheta}$ for the tuning parameter $\kappa$,
we have:
\begin{equation}
\frac{1}{2\nu(\spi)}\norm{Y - \hat{B}\Phi_T(\hat{\vartheta})}_{L_T}^2 + \kappa \|\hat{B}\|_{\ell_1, L^p(\nu)} \leq \frac{1}{2\nu(\spi)}\norm{Y - B^\star\Phi_T(\vartheta^\star)}_{L_T}^2 + \kappa \|B^\star\|_{\ell_1, L^p(\nu)}.
\label{eq:optim}
\end{equation}
We define the linear mapping $\hat{\Upsilon} $ from $L_T$ to $\R$ by:
\[
\hat{\Upsilon}( F)= \left \langle \hat{B}
\Phi_{T}(\hat{\vartheta})-
B^{\star}\Phi_{T}(\vartheta^{\star}),F \right \rangle_{L_T}.
\]
This gives,  by rearranging terms  and using  the equation of  the model
$Y = B^\star\Phi_T(\vartheta^\star) + W_T$, that:
\begin{equation}
\label{eq:optim_bis}
\frac{1}{2}\hat{R}_T^2
\leq \frac{1}{\nu(\spi)}\hat{\Upsilon} (W_T) + \kappa \left ( \|B^{\star}\|_{\ell_1, L^p(\nu)} -
\|\hat{B}\|_{\ell_1, L^p(\nu)}\right ). 
\end{equation}
Next,  we  shall  expand  the  two  terms on  the  right-hand  side  of
\eqref{eq:optim_bis}. Recall the subset $\cq^\star=\{ \theta^\star _k\, \colon\,  k\in S^\star\}$ of $\Theta_T$  is the set of the active non-linear parameters of the model. In the
rest of the proof, we fix  $r>0$ so that Assumptions \ref{assumption1} and
\ref{assumption2} are verified for $\cq^{\star}$ and bound them from above. 

In particular, for all
$k\neq k' $ in the support $ S^\star = \{k, \, \norm{B_{k}^{\star}}_{L^2(\nu)} \neq 0 \}$,  we         have
$\mathfrak{d}_T(\theta_k^\star,\theta_{k'}^\star) > 2r$.

Let us define the following sets  of indices:
\begin{itemize}
	\item [-]$\hat{S} = \left\{\ell : \norm{\hat{B}_{\ell}}_{L^p(\nu)} \neq 0\right\}$ the support set of $\hat{B}$ given by the optimization problem \eqref{eq:generalized_lasso};
	\item [-]$\tilde{S}_{k}(r) = \left \{\ell\in \hat{S}: \mathfrak{d}_T(\hat{\theta}_{\ell},\theta_{k}^{\star}) \leq r \right \}$ the set of indices $\ell$ in the support of $\hat B$ associated to the active parametric functions having $\hat \theta_\ell$ close to the true parameter $\theta_k^{\star}$, for a fixed $k$ in $S^\star$;		
	\item [-] $\tilde{S}(r) =  \bigcup_{k \in S^\star} \tilde{S}_{k}(r) $  the set of indices $\ell$ in the support of $\hat B$ associated to the active parametric functions having $\hat \theta_\ell$ close to any true parameter $\theta_k^{\star}$, for some $k$ in $S^\star$.
\end{itemize}
Since      the       closed      balls
$\mathcal{B}_T(\theta^{\star}_{k  },r)$  with   $k  \in  S^{\star}$  are
pairwise disjoint, the  sets $\tilde{S}_k(r)$, for $k  \in S^\star$, are
also pairwise disjoint and one can write the following decomposition with $\tilde{S}(r)^{c} = \{1,\cdots,K\} \setminus \tilde{S}(r)$:
\begin{equation*} \label{eq:decomposition_balls}
\begin{aligned}
\hat{B}\Phi_{T}(\hat{\vartheta}) - B^{\star}\Phi_{T}(\vartheta^{\star}) & =  \sum\limits_{k=1}^{K} \hat{B}_{k}\phi_{T}(\hat{\theta}_{k}) - \sum\limits_{k\in S^{\star}}B_{k}^{\star}\phi_{T}(\theta^{\star}_{k}) \\
&=  \sum\limits_{k \in S^{\star}, \tilde{S}_{k}(r) \neq \emptyset} \sum\limits_{\ell \in \tilde{S}_{k}(r) }\hat{B}_{\ell}\phi_{T}(\hat{\theta}_{\ell}) + \sum\limits_{k \in \tilde{S}(r)^{c}}\hat{B}_{k}\phi_{T}(\hat{\theta}_{k}) -  \sum\limits_{k\in S^{\star}}B_{k}^{\star}\phi_{T}(\theta^{\star}_{k}).
\end{aligned}
\end{equation*} 

This  decomposition  groups  the   elements  of  the  predicted  mixture
according to the proximity of the estimated parameter $\hat \theta_\ell$
to a true underlying parameter $\theta_k^\star$ to be estimated. We 
use   a   Taylor-type   expansion  with   the   Riemann   distance
$\mathfrak{d}_T$ for  the function $\phi_T(\theta)$ around  the elements
of $\cq^\star$. By assumption, the function $\phi_T$ is
twice continuously differentiable with  respect to the variable $\theta$
and  the  function  $g_T$ is  positive  on
$\Theta_T$.  We  recall the notation 
$ \phi_{T}^{[i]}=\tilde{D}_{i;T}[\phi_{T}] $ for  $i\in \{0, 1, 2\}$.  According
to   \cite[Lemma 4.2]{butucea22}, we have for  any $\theta_k^{\star}$ and
$\hat{\theta}_{\ell}$ in $\Theta_T$:
\[
\phi_{T}(\hat{\theta}_\ell) = \phi_{T}(\theta_k^{\star}) +
\operatorname{sign}(\hat{\theta}_\ell-\theta_k^{\star})\,
\mathfrak{d}_T(\hat{\theta}_\ell,\theta_k^{\star})\,
\phi_{T}^{[1]}(\theta_k^{\star}) +
\mathfrak{d}_T(\hat{\theta}_\ell,\theta_k^{\star})^2 \,  \int_0^1 (1-s)
\phi_{T}^{[2]}(\gamma_s^{(k\ell)}) \, \rd s, 
\]
where $\gamma^{(k\ell)}$ is a distance realizing geodesic path belonging to $\Theta_T$ such that $\gamma_0^{(k\ell)} = \theta_{k}^{\star}$, $\gamma_1^{(k\ell)}= \hat \theta_\ell$ and $\mathfrak{d}_T(\hat \theta_\ell,\theta_k^\star) = \int_{0}^1   | \dot\gamma_s^{(k\ell)}| \sqrt{g_T(\gamma_s^{(k\ell)})} \rd s$. 
Hence we obtain:
\begin{multline}
\label{eq:decomp_taylor}
\hat{B}\Phi_{T}(\hat{\vartheta}) -
B^{\star}\Phi_{T}(\vartheta^{\star})
=    \sum\limits_{k \in S^{\star}} I_{0,k}(r)\, 
\phi_{T}(\theta_{k}^{\star}) + \sum\limits_{k \in S^{\star}} I_{1, k}
(r)\, \phi_{T}^{[1]}(\theta_k^{\star})
+ \sum\limits_{\ell \in \tilde{S}(r)^{c}}\hat{B}_{\ell} \, \phi_{T}(\hat{\theta}_{\ell })\\ 
\quad+ \sum\limits_{k \in S^{\star}} \left( \sum\limits_{\ell \in \tilde{S}_{k}(r) } \hat{B}_{\ell}  \, \mathfrak{d}_T(\hat{\theta}_{\ell},\theta_k^{\star})^2 \, \int_0^1 (1-s) \phi_{T}^{[2]}(\gamma_s^{(k\ell)}) \, \rd s  \right ),
\end{multline}
with
\begin{equation*}
\label{eq:def-I0-I1}
I_{0, k}(r)= \left (\sum\limits_{\ell \in \tilde{S}_{k}(r) } \hat{B}_{\ell} \right )- B_{k}^{\star} 
\quad\text{and}\quad
I_{1, k} (r)= \sum\limits_{\ell \in \tilde{S}_{k}(r) } \hat{B}_{\ell} \, \operatorname{sign}(\hat{\theta}_{\ell}-\theta_k^{\star})\, \mathfrak{d}_T(\hat{\theta}_{\ell},\theta_k^{\star}). 
\end{equation*}

We note that $ I_{0, k}(r)$ and $ I_{1, k}(r)$ are functions of $z$ that belong to $L^2(\nu)$. We shall omit the dependence in $r$ and in $z$ when there is no ambiguity. 
Let us moreover denote by:
\begin{align}
I_0(r)
&= \sum\limits_{k \in S^{\star}} \norm{I_{0,k}(r)}_{L^p(\nu)}
\qquad \text{ and } \qquad
I_1(r)  = \sum\limits_{k \in S^{\star}} \norm{I_{1,k}(r)}_{L^p(\nu)}, \nonumber\\
\label{eq:I2}  
I_{2,k}(r)
&= \sum\limits_{\ell \in \tilde{S}_{k}(r) } \norm{\hat{B}_{\ell}}_{L^p(\nu)}
\mathfrak{d}_T(\hat{\theta}_{\ell},\theta_k^{\star})^2
\qquad \text{ and }\qquad
I_2(r)  = \sum\limits_{k \in S^{\star}} I_{2,k}(r),\\
\label{eq:I3}  
I_3(r) & = \sum\limits_{\ell \in \tilde{S}(r)^{c}} \norm{\hat{B}_{\ell}}_{L^p(\nu)} = \left \|\hat{B}_{\tilde{S}(r)^{c}}\right
\|_{\ell_1, L^p(\nu)},
\end{align}
where $\hat{B}_{\tilde{S}(r)^{c}}$ denotes the restriction of the vector-valued mapping $\hat{B}$ to its components in the set of indices $\tilde{S}(r)^{c}$. Again, we omit the dependence in $r$ when there is no ambiguity.

Let us bound now the difference $\norm{B^{\star}}_{\ell_1, L^p(\nu)}-\norm{\hat{B}}_{\ell_1, L^p(\nu)}$, see \eqref{eq:optim_bis}, by using Lemma A.1 of the supplementary material: 
\begin{eqnarray}
\norm{B^{\star}}_{\ell_1, L^p(\nu)}-\norm{\hat{B}}_{\ell_1, L^p(\nu)}
&=& \sum\limits_{k \in S^{\star}} \Big (\norm{B_{k}^\star}_{L^p(\nu)} -
\sum\limits_{\ell \in \tilde{S}_{k}(r) } \norm{\hat{B}_{\ell}}_{L^p(\nu)} \Big )
- \sum\limits_{k \in \tilde{S}(r)^{c}} \norm{\hat{B}_{k}}_{L^p(\nu)} \nonumber\\
&\leq &
I_0 \leq 
 C_{N}' I_2 + (1 - C_{F})I_3 + | \hat{\Upsilon}(P_1) |, \label{eq:diff_l1_norm} 
\end{eqnarray}
where the positive constants $C_N', \, C_F<1$ are given in Assumption~\ref{assumption1} and $P_1\in H_T$ corresponds to the certificate $P$ therein with $V$  given in Lemma A.1 of the supplementary material.

We give next an upper bound for $ | \hat{\Upsilon}(W_T) | = |\langle \hat{B}\Phi_{T}(\hat{\vartheta}) - B^{\star}\Phi_{T}(\vartheta^{\star})   , W_T\rangle_{L_T}|$ in \eqref{eq:optim_bis}. First, we use the expansion \eqref{eq:decomp_taylor} and Hölder's inequality and we get as an example for the first term:
\begin{align*}
   |\langle  \sum\limits_{k \in S^{\star}} I_{0,k}(r)\, 
\phi_{T}(\theta_{k}^{\star}) , W_T \rangle_{L_T}|
&\leq  \sum\limits_{k \in S^{\star}}|\langle  I_{0,k}(r)\, 
\phi_{T}(\theta_{k}^{\star})  , W_T \rangle_{L_T}| \\
&\leq 
 \sum\limits_{k \in S^{\star}} \|  I_{0,k}(r)\|_{L^p(\nu)} \cdot 
\| \langle \phi_{T}(\theta_{k}^{\star}),W_T \rangle_T\|_{L^q(\nu)}\\
&\leq I_0(r) \cdot \sup_{\theta \in \Theta_T} \| \langle \phi_{T}(\theta),W_T \rangle_T\|_{L^q(\nu)} = I_0 \cdot M_0,
\end{align*}
where the random variables $M_i$ for $i \in \{0,1,2\}$ are defined in \eqref{def:M}. 
We proceed similarly for the remaining terms to get that: \begin{align}
\nonumber 
| \hat{\Upsilon}(W_T) |
&\leq (I_0 + I_3) M_{0} + I_1 M_{1} + I_2 \, 2^{-1} \, M_{2}\\
&\leq (C_{N}' I_{2 } + (2 - C_{F})I_3 + |\hat{\Upsilon}(P_1)|) M_{0}
+ (c_{N}I_2 + c_{F}I_3 + |\hat{\Upsilon}(Q_0)|)M_{1} + I_2 \,
2^{-1} \, M_{2},
\label{eq:w.diff}
\end{align}
where we also applied Lemmas A.1 and A.2 of the supplementary material, with the positive constants $C_N',\, C_F, \,c_N, \, c_F$ given in Assumptions~\ref{assumption1} and \ref{assumption2} and $Q_0\in L_T$ corresponds to the derivative certificate $Q$ in Assumption~\ref{assumption2} with $V$ given in Lemma A.2 of the supplementary material.
By reinjecting \eqref{eq:diff_l1_norm} and \eqref{eq:w.diff} in \eqref{eq:optim_bis} one gets:
\begin{multline*}
\frac{1}{2}\hat{R}_T^2
\leq I_2\left (\frac{C_{N}'M_{0} + c_{N}M_{1} + 2^{-1} M_{2}}{\nu(\spi)} + \kappa C_{N}' \right ) + I_3 \left(
\frac{(2 - C_{F})M_{0}+ c_{F} M_{1}}{\nu(\spi)} + \kappa (1 - C_{F}) \right )\\ 
+|\hat{\Upsilon}(P_1)| \left (\frac{M_{0}}{\nu(\spi)}+ \kappa \right) + |\hat{\Upsilon}(Q_0)|\frac{M_{1}}{\nu(\spi)}\cdot  
\end{multline*}
We define the events:
\begin{equation}
\label{eq:events_A}
\mathcal{A}_i = \left \{ \frac 1{ \nu(\spi)} M_{i} \leq  \mathcal{C} \, \kappa \, \right \},
\, \quad\text{for $i \in \{0,1,2\}$}
\quad \text{and}\quad
\mathcal{A} = \mathcal{A}_0 \cap \mathcal{A}_1 \cap \mathcal{A}_2,
\end{equation}
where:
$
\mathcal{C}=  \frac {C_F }{2 ( 2-C_F + c_F)} \wedge
\frac{C_N}{ 2  ( C'_{N} + c_{N} + 2^{-1})}\cdot 
$
Using Lemma A.5 of the supplementary material, we obtain an upper bound
for the prediction error on the event $\mathcal{A}$:
\begin{equation}
\label{eq:bound_A_square}
\hat{R}_T^2 \leq \mathcal{C}''  \, \kappa \, (|\hat{\Upsilon}(P_0)| + |\hat{\Upsilon}(P_1)| + |\hat{\Upsilon}(Q_0)|),
\end{equation}
with  $P_0\in H_T$ corresponding to the certificate $P$ in Assumption~\ref{assumption1} with $V$ given by (50) of the supplementary material and:
$
\mathcal{C}'' = 4{\mathcal{C}'} \left(1+
\frac{\mathcal{C}'}{C_N}( 2 C_N'+c_N+1)
+ \frac{\mathcal{C}'}{C_F}(3-2C_F + c_F)
\right) \quad \text{and} \quad \mathcal{C}' = \mathcal{C} \vee 1.
$
Using the Cauchy-Schwarz inequality and the definition of
$\hat{\Upsilon}$, we get that  for $ f \in L_T$:
$
\label{eq:functional_bound}
|\hat{\Upsilon}(f) |\leq \hat{R}_T \, \sqrt{\nu(\spi)} \, \norm{f}_{L_T}.$

Using Assumption~$\ref{assumption1}~\ref{it:norm<c}$ for $P_i$ with $i=1,2$,
and Assumption~$\ref{assumption2}~\ref{it:as2-<c}$ for $Q_0$,  we get:
$
\label{eq:bounds_pq}
\norm{P_i}_{L_T}  \leq C_B\sqrt{\sparse}\nu(\spi)^{1/2p-1/2q}
\quad \text{and}\quad
\norm{Q_0}_{L_T} \leq c_B\sqrt{\sparse}\nu(\spi)^{1/2p-1/2q}.
$
Plugging this in \eqref{eq:bound_A_square}, we get that on the event $\mathcal{A}$:
$
\label{eq:bound_prediction_squared}
\hat{R}_T^2 \leq\sqrt{ \mathcal{C}_0} \, \kappa \hat{R}_T \,  \sqrt{\sparse}\, \nu(\spi)^{\frac{1}{p}}
\quad\text{with}\quad \mathcal{C}_0 = (c_B+2C_B)^2 \mathcal{C}''^2.
$
We obtain \eqref{eq:main_theorem} on the event $\mathcal{A}$ defined in \eqref{eq:events_A} whose probability writes as in \eqref{eq:proba_maintheorem}.

%%%%%%%%%%%%%%%%%%%%%%%%%%%%%%%%%%%%%%%%%%%%%%
%% Single Appendix:                         %%
%%%%%%%%%%%%%%%%%%%%%%%%%%%%%%%%%%%%%%%%%%%%%%
%\begin{appendix}
%\section*{???}%% if no title is needed, leave empty \section*{}.
%\end{appendix}
%%%%%%%%%%%%%%%%%%%%%%%%%%%%%%%%%%%%%%%%%%%%%%
%% Multiple Appendixes:                     %%
%%%%%%%%%%%%%%%%%%%%%%%%%%%%%%%%%%%%%%%%%%%%%%
%\begin{appendix}
%\section{???}
%
%\section{???}
%
%\end{appendix}

%%%%%%%%%%%%%%%%%%%%%%%%%%%%%%%%%%%%%%%%%%%%%%
%% Support information, if any,             %%
%% should be provided in the                %%
%% Acknowledgements section.                %%
%%%%%%%%%%%%%%%%%%%%%%%%%%%%%%%%%%%%%%%%%%%%%%
\begin{acks}[Acknowledgments]
This work was partially supported by the ANRT grant
N°2019/1260 and the grant Investissements d’Avenir (ANR11-IDEX0003/Labex Ecodec/ANR-11-LABX-00). 
The authors are grateful to the associate editor and the referees for their useful comments.
\end{acks}
%%%%%%%%%%%%%%%%%%%%%%%%%%%%%%%%%%%%%%%%%%%%%%
%% Funding information, if any,             %%
%% should be provided in the                %%
%% funding section.                         %%
%%%%%%%%%%%%%%%%%%%%%%%%%%%%%%%%%%%%%%%%%%%%%%
%\begin{funding}
% The first author was supported by ...
%
% The second author was supported in part by ...
%\end{funding}

%%%%%%%%%%%%%%%%%%%%%%%%%%%%%%%%%%%%%%%%%%%%%%
%% Supplementary Material, including data   %%
%% sets and code, should be provided in     %%
%% {supplement} environment with title      %%
%% and short description. It cannot be      %%
%% available exclusively as external link.  %%
%% All Supplementary Material must be       %%
%% available to the reader on Project       %%
%% Euclid with the published article.       %%
%%%%%%%%%%%%%%%%%%%%%%%%%%%%%%%%%%%%%%%%%%%%%%
\begin{supplement}
All proofs not included in this text can be found in the supplementary material \cite{butuceaSupplement}.
\end{supplement}

%%  MR numbers will be added by VTeX.                      %%
%%                                                         %%
%%  Use \cite{...} to cite references in text.             %%
%%                                                         %%
%%%%%%%%%%%%%%%%%%%%%%%%%%%%%%%%%%%%%%%%%%%%%%%%%%%%%%%%%%%%%

%% if your bibliography is in bibtex format, uncomment commands:
\bibliographystyle{imsart-nameyear} % Style BST file (imsart-number.bst or imsart-nameyear.bst)
\bibliography{ref.bib}       % Bibliography file (usually '*.bib')

\begin{thebibliography}{38}
% BibTex style file: imsart-nameyear.bst, 2017-11-03
% Default style options (sort=1,type=nameyear).
% Used options (sort=1,type=nameyear).

\bibitem[\protect\citeauthoryear{Aza\"{\i}s and
  Wschebor}{2009}]{azais2009level}
\begin{bbook}[author]
\bauthor{\bsnm{Aza\"{\i}s},~\bfnm{Jean-Marc}\binits{J.-M.}} \AND
  \bauthor{\bsnm{Wschebor},~\bfnm{Mario}\binits{M.}}
(\byear{2009}).
\btitle{Level Sets and Extrema of Random Processes and Fields}.
\bpublisher{John Wiley \& Sons, Inc., Hoboken, NJ}.
\end{bbook}
\endbibitem

\bibitem[\protect\citeauthoryear{Bach}{2008}]{Bach08}
\begin{barticle}[author]
\bauthor{\bsnm{Bach},~\bfnm{Francis~R.}\binits{F.~R.}}
(\byear{2008}).
\btitle{Consistency of the group lasso and multiple kernel learning}.
\bjournal{J. Mach. Learn. Res.}
\bvolume{9}
\bpages{1179--1225}.
\end{barticle}
\endbibitem

\bibitem[\protect\citeauthoryear{Barber, Reimherr and Schill}{2017}]{Barber17}
\begin{barticle}[author]
\bauthor{\bsnm{Barber},~\bfnm{Rina~Foygel}\binits{R.~F.}},
  \bauthor{\bsnm{Reimherr},~\bfnm{Matthew}\binits{M.}} \AND
  \bauthor{\bsnm{Schill},~\bfnm{Thomas}\binits{T.}}
(\byear{2017}).
\btitle{The function-on-scalar lasso with applications to longitudinal {GWAS}}.
\bjournal{Electron. J. Stat.}
\bvolume{11}
\bpages{1351--1389}.
\end{barticle}
\endbibitem

\bibitem[\protect\citeauthoryear{Beck and Teboulle}{2009}]{Beck2009}
\begin{barticle}[author]
\bauthor{\bsnm{Beck},~\bfnm{Amir}\binits{A.}} \AND
  \bauthor{\bsnm{Teboulle},~\bfnm{Marc}\binits{M.}}
(\byear{2009}).
\btitle{A fast iterative shrinkage-thresholding algorithm for linear inverse
  problems}.
\bjournal{SIAM J. Imaging Sci.}
\bvolume{2}
\bpages{183--202}.
\end{barticle}
\endbibitem

\bibitem[\protect\citeauthoryear{Bhaskar, Tang and
  Recht}{2013}]{bhaskar2013atomic}
\begin{barticle}[author]
\bauthor{\bsnm{Bhaskar},~\bfnm{Badri~Narayan}\binits{B.~N.}},
  \bauthor{\bsnm{Tang},~\bfnm{Gongguo}\binits{G.}} \AND
  \bauthor{\bsnm{Recht},~\bfnm{Benjamin}\binits{B.}}
(\byear{2013}).
\btitle{Atomic norm denoising with applications to line spectral estimation}.
\bjournal{IEEE Trans. Signal Process.}
\bvolume{61}
\bpages{5987--5999}.
\end{barticle}
\endbibitem

\bibitem[\protect\citeauthoryear{Bickel, Ritov and
  Tsybakov}{2009}]{bickel2009simultaneous}
\begin{barticle}[author]
\bauthor{\bsnm{Bickel},~\bfnm{Peter~J.}\binits{P.~J.}},
  \bauthor{\bsnm{Ritov},~\bfnm{Ya'acov}\binits{Y.}} \AND
  \bauthor{\bsnm{Tsybakov},~\bfnm{Alexandre~B.}\binits{A.~B.}}
(\byear{2009}).
\btitle{Simultaneous analysis of lasso and dantzig selector}.
\bjournal{Ann. Statist.}
\bvolume{37}
\bpages{1705--1732}.
\end{barticle}
\endbibitem

\bibitem[\protect\citeauthoryear{Boyd, Schiebinger and
  Recht}{2017}]{boyd2017alternating}
\begin{barticle}[author]
\bauthor{\bsnm{Boyd},~\bfnm{Nicholas}\binits{N.}},
  \bauthor{\bsnm{Schiebinger},~\bfnm{Geoffrey}\binits{G.}} \AND
  \bauthor{\bsnm{Recht},~\bfnm{Benjamin}\binits{B.}}
(\byear{2017}).
\btitle{The alternating descent conditional gradient method for sparse inverse
  problems}.
\bjournal{SIAM J. Optim.}
\bvolume{27}
\bpages{616--639}.
\end{barticle}
\endbibitem

\bibitem[\protect\citeauthoryear{Boyer, De~Castro and
  Salmon}{2017}]{boyer2017adapting}
\begin{barticle}[author]
\bauthor{\bsnm{Boyer},~\bfnm{Claire}\binits{C.}},
  \bauthor{\bsnm{De~Castro},~\bfnm{Yohann}\binits{Y.}} \AND
  \bauthor{\bsnm{Salmon},~\bfnm{Joseph}\binits{J.}}
(\byear{2017}).
\btitle{Adapting to unknown noise level in sparse deconvolution}.
\bjournal{Inf. Inference}
\bvolume{6}
\bpages{310--348}.
\end{barticle}
\endbibitem

\bibitem[\protect\citeauthoryear{Boyer et~al.}{2019}]{boyer2019representer}
\begin{barticle}[author]
\bauthor{\bsnm{Boyer},~\bfnm{Claire}\binits{C.}},
  \bauthor{\bsnm{Chambolle},~\bfnm{Antonin}\binits{A.}},
  \bauthor{\bsnm{De~Castro},~\bfnm{Yohann}\binits{Y.}},
  \bauthor{\bsnm{Duval},~\bfnm{Vincent}\binits{V.}}, \bauthor{\bparticle{de}
  \bsnm{Gournay},~\bfnm{Fr\'{e}d\'{e}ric}\binits{F.}} \AND
  \bauthor{\bsnm{Weiss},~\bfnm{Pierre}\binits{P.}}
(\byear{2019}).
\btitle{On representer theorems and convex regularization}.
\bjournal{SIAM J. Optim.}
\bvolume{29}
\bpages{1260--1281}.
\end{barticle}
\endbibitem

\bibitem[\protect\citeauthoryear{B\"{u}hlmann and van~de
  Geer}{2011}]{buhlmann2011statistics}
\begin{bbook}[author]
\bauthor{\bsnm{B\"{u}hlmann},~\bfnm{Peter}\binits{P.}} \AND
  \bauthor{\bparticle{van~de} \bsnm{Geer},~\bfnm{Sara}\binits{S.}}
(\byear{2011}).
\btitle{Statistics for High-Dimensional Data}.
\bseries{Springer Series in Statistics}.
\bpublisher{Springer, Heidelberg}
\bnote{Methods, theory and applications}.
\end{bbook}
\endbibitem

\bibitem[\protect\citeauthoryear{Bunea, Tsybakov and
  Wegkamp}{2007}]{bunea2007sparsity}
\begin{barticle}[author]
\bauthor{\bsnm{Bunea},~\bfnm{Florentina}\binits{F.}},
  \bauthor{\bsnm{Tsybakov},~\bfnm{Alexandre}\binits{A.}} \AND
  \bauthor{\bsnm{Wegkamp},~\bfnm{Marten}\binits{M.}}
(\byear{2007}).
\btitle{Sparsity oracle inequalities for the lasso}.
\bjournal{Electron. J. Stat.}
\bvolume{1}
\bpages{169--194}.
\end{barticle}
\endbibitem

\bibitem[\protect\citeauthoryear{Butucea et~al.}{2021}]{butucea2021}
\begin{binproceedings}[author]
\bauthor{\bsnm{Butucea},~\bfnm{Cristina}\binits{C.}},
  \bauthor{\bsnm{Delmas},~\bfnm{Jean-François}\binits{J.-F.}},
  \bauthor{\bsnm{Dutfoy},~\bfnm{Anne}\binits{A.}} \AND
  \bauthor{\bsnm{Hardy},~\bfnm{Cl\'{e}ment}\binits{C.}}
(\byear{2021}).
\btitle{Modeling infra-red spectra: an algorithm for an automatic and
  simultaneous analysis}.
In \bbooktitle{In Proceedings of the 31st European Safety and Reliability
  Conference}
\bpages{3359--3366}.
\end{binproceedings}
\endbibitem

\bibitem[\protect\citeauthoryear{Butucea et~al.}{2022}]{butucea22}
\begin{barticle}[author]
\bauthor{\bsnm{Butucea},~\bfnm{Cristina}\binits{C.}},
  \bauthor{\bsnm{Delmas},~\bfnm{Jean-Fran{\c{c}}ois}\binits{J.-F.}},
  \bauthor{\bsnm{Dutfoy},~\bfnm{Anne}\binits{A.}} \AND
  \bauthor{\bsnm{Hardy},~\bfnm{Cl{\'e}ment}\binits{C.}}
(\byear{2022}).
\btitle{Off-the-grid learning of sparse mixtures from a continuous dictionary}.
\bjournal{arXiv preprint arXiv:2207.00171}.
\end{barticle}
\endbibitem

\bibitem[\protect\citeauthoryear{Butucea et~al.}{2023}]{butuceaSupplement}
\begin{barticle}[author]
\bauthor{\bsnm{Butucea},~\bfnm{Cristina}\binits{C.}},
  \bauthor{\bsnm{Delmas},~\bfnm{Jean-Fran{\c{c}}ois}\binits{J.-F.}},
  \bauthor{\bsnm{Dutfoy},~\bfnm{Anne}\binits{A.}} \AND
  \bauthor{\bsnm{Hardy},~\bfnm{Cl{\'e}ment}\binits{C.}}
(\byear{2023}).
\btitle{Supplement to ``Simultaneous off-the-grid learning of sparse mixtures
  from a continuous dictionary''}.
\end{barticle}
\endbibitem

\bibitem[\protect\citeauthoryear{Cand\`es and
  Fernandez-Granda}{2013}]{candes2013super}
\begin{barticle}[author]
\bauthor{\bsnm{Cand\`es},~\bfnm{Emmanuel~J.}\binits{E.~J.}} \AND
  \bauthor{\bsnm{Fernandez-Granda},~\bfnm{Carlos}\binits{C.}}
(\byear{2013}).
\btitle{Super-resolution from noisy data}.
\bjournal{J. Fourier Anal. Appl.}
\bvolume{19}
\bpages{1229--1254}.
\end{barticle}
\endbibitem

\bibitem[\protect\citeauthoryear{Cand\`es and
  Fernandez-Granda}{2014}]{candes2014towards}
\begin{barticle}[author]
\bauthor{\bsnm{Cand\`es},~\bfnm{Emmanuel~J.}\binits{E.~J.}} \AND
  \bauthor{\bsnm{Fernandez-Granda},~\bfnm{Carlos}\binits{C.}}
(\byear{2014}).
\btitle{Towards a mathematical theory of super-resolution}.
\bjournal{Comm. Pure Appl. Math.}
\bvolume{67}
\bpages{906--956}.
\end{barticle}
\endbibitem

\bibitem[\protect\citeauthoryear{Cand\`es and
  Plan}{2011}]{candes2011probabilistic}
\begin{barticle}[author]
\bauthor{\bsnm{Cand\`es},~\bfnm{Emmanuel~J.}\binits{E.~J.}} \AND
  \bauthor{\bsnm{Plan},~\bfnm{Yaniv}\binits{Y.}}
(\byear{2011}).
\btitle{A probabilistic and {RIP}less theory of compressed sensing}.
\bjournal{IEEE Trans. Inform. Theory}
\bvolume{57}
\bpages{7235--7254}.
\end{barticle}
\endbibitem

\bibitem[\protect\citeauthoryear{Candes and Tao}{2007}]{candes2007dantzig}
\begin{barticle}[author]
\bauthor{\bsnm{Candes},~\bfnm{Emmanuel}\binits{E.}} \AND
  \bauthor{\bsnm{Tao},~\bfnm{Terence}\binits{T.}}
(\byear{2007}).
\btitle{The dantzig selector: statistical estimation when {$p$} is much larger
  than {$n$}}.
\bjournal{Ann. Statist.}
\bvolume{35}
\bpages{2313--2351}.
\end{barticle}
\endbibitem

\bibitem[\protect\citeauthoryear{Chesneau and Hebiri}{2008}]{Chesneau08}
\begin{barticle}[author]
\bauthor{\bsnm{Chesneau},~\bfnm{Ch.}\binits{C.}} \AND
  \bauthor{\bsnm{Hebiri},~\bfnm{M.}\binits{M.}}
(\byear{2008}).
\btitle{Some theoretical results on the grouped variables lasso}.
\bjournal{Math. Methods Statist.}
\bvolume{17}
\bpages{317--326}.
\end{barticle}
\endbibitem

\bibitem[\protect\citeauthoryear{Chizat}{2021}]{chizat2021sparse}
\begin{barticle}[author]
\bauthor{\bsnm{Chizat},~\bfnm{Lenaic}\binits{L.}}
(\byear{2021}).
\btitle{Sparse optimization on measures with over-parameterized gradient
  descent}.
\bjournal{Mathematical Programming}
\bpages{1--46}.
\end{barticle}
\endbibitem

\bibitem[\protect\citeauthoryear{de~Castro and Gamboa}{2012}]{de2012exact}
\begin{barticle}[author]
\bauthor{\bparticle{de} \bsnm{Castro},~\bfnm{Yohann}\binits{Y.}} \AND
  \bauthor{\bsnm{Gamboa},~\bfnm{Fabrice}\binits{F.}}
(\byear{2012}).
\btitle{Exact reconstruction using beurling minimal extrapolation}.
\bjournal{J. Math. Anal. Appl.}
\bvolume{395}
\bpages{336--354}.
\end{barticle}
\endbibitem

\bibitem[\protect\citeauthoryear{Denoyelle et~al.}{2020}]{denoyelle2019sliding}
\begin{barticle}[author]
\bauthor{\bsnm{Denoyelle},~\bfnm{Quentin}\binits{Q.}},
  \bauthor{\bsnm{Duval},~\bfnm{Vincent}\binits{V.}},
  \bauthor{\bsnm{Peyr\'{e}},~\bfnm{Gabriel}\binits{G.}} \AND
  \bauthor{\bsnm{Soubies},~\bfnm{Emmanuel}\binits{E.}}
(\byear{2020}).
\btitle{The sliding frank-wolfe algorithm and its application to
  super-resolution microscopy}.
\bjournal{Inverse Problems}
\bvolume{36}
\bpages{014001, 42}.
\end{barticle}
\endbibitem

\bibitem[\protect\citeauthoryear{Diestel and Uhl}{1977}]{diestel}
\begin{bbook}[author]
\bauthor{\bsnm{Diestel},~\bfnm{J.}\binits{J.}} \AND
  \bauthor{\bsnm{Uhl},~\bfnm{J.~J.}\binits{J.~J.} \bsuffix{Jr.}}
(\byear{1977}).
\btitle{Vector Measures}.
\bseries{Mathematical Surveys, No. 15}.
\bpublisher{American Mathematical Society, Providence, R.I.}
\bnote{With a foreword by B. J. Pettis}.
\end{bbook}
\endbibitem

\bibitem[\protect\citeauthoryear{Duval}{2021}]{duval21}
\begin{barticle}[author]
\bauthor{\bsnm{Duval},~\bfnm{Vincent}\binits{V.}}
(\byear{2021}).
\btitle{An epigraphical approach to the representer theorem}.
\bjournal{J. Convex Anal.}
\bvolume{28}
\bpages{819--836}.
\end{barticle}
\endbibitem

\bibitem[\protect\citeauthoryear{Duval and Peyr\'{e}}{2015}]{duval2015exact}
\begin{barticle}[author]
\bauthor{\bsnm{Duval},~\bfnm{Vincent}\binits{V.}} \AND
  \bauthor{\bsnm{Peyr\'{e}},~\bfnm{Gabriel}\binits{G.}}
(\byear{2015}).
\btitle{Exact support recovery for sparse spikes deconvolution}.
\bjournal{Found. Comput. Math.}
\bvolume{15}
\bpages{1315--1355}.
\end{barticle}
\endbibitem

\bibitem[\protect\citeauthoryear{Duval and Peyr\'{e}}{2017}]{duval2017thingrid}
\begin{barticle}[author]
\bauthor{\bsnm{Duval},~\bfnm{Vincent}\binits{V.}} \AND
  \bauthor{\bsnm{Peyr\'{e}},~\bfnm{Gabriel}\binits{G.}}
(\byear{2017}).
\btitle{Sparse regularization on thin grids {I}: the lasso}.
\bjournal{Inverse Problems}
\bvolume{33}
\bpages{055008, 29}.
\end{barticle}
\endbibitem

\bibitem[\protect\citeauthoryear{Golbabaee and Poon}{2022}]{golbabaee2020off}
\begin{barticle}[author]
\bauthor{\bsnm{Golbabaee},~\bfnm{Mohammad}\binits{M.}} \AND
  \bauthor{\bsnm{Poon},~\bfnm{Clarice}\binits{C.}}
(\byear{2022}).
\btitle{An off-the-grid approach to multi-compartment magnetic resonance
  fingerprinting}.
\bjournal{Inverse Problems}
\bvolume{38}
\bpages{Paper No. 085002, 31}.
\end{barticle}
\endbibitem

\bibitem[\protect\citeauthoryear{Huang and Zhang}{2010}]{Huang10}
\begin{barticle}[author]
\bauthor{\bsnm{Huang},~\bfnm{Junzhou}\binits{J.}} \AND
  \bauthor{\bsnm{Zhang},~\bfnm{Tong}\binits{T.}}
(\byear{2010}).
\btitle{The benefit of group sparsity}.
\bjournal{Ann. Statist.}
\bvolume{38}
\bpages{1978--2004}.
\end{barticle}
\endbibitem

\bibitem[\protect\citeauthoryear{Liu and Zhang}{2008}]{liu2008}
\begin{barticle}[author]
\bauthor{\bsnm{Liu},~\bfnm{Han}\binits{H.}} \AND
  \bauthor{\bsnm{Zhang},~\bfnm{Jian}\binits{J.}}
(\byear{2008}).
\btitle{On the $\ell_1$-$\ell_q$ regularized regression}.
\bjournal{arXiv preprint arXiv:0802.1517}.
\end{barticle}
\endbibitem

\bibitem[\protect\citeauthoryear{Lounici et~al.}{2011}]{lounici2011oracle}
\begin{barticle}[author]
\bauthor{\bsnm{Lounici},~\bfnm{Karim}\binits{K.}},
  \bauthor{\bsnm{Pontil},~\bfnm{Massimiliano}\binits{M.}},
  \bauthor{\bparticle{van~de} \bsnm{Geer},~\bfnm{Sara}\binits{S.}} \AND
  \bauthor{\bsnm{Tsybakov},~\bfnm{Alexandre~B.}\binits{A.~B.}}
(\byear{2011}).
\btitle{Oracle inequalities and optimal inference under group sparsity}.
\bjournal{Ann. Statist.}
\bvolume{39}
\bpages{2164--2204}.
\end{barticle}
\endbibitem

\bibitem[\protect\citeauthoryear{Nardi and Rinaldo}{2008}]{Nardi08}
\begin{barticle}[author]
\bauthor{\bsnm{Nardi},~\bfnm{Yuval}\binits{Y.}} \AND
  \bauthor{\bsnm{Rinaldo},~\bfnm{Alessandro}\binits{A.}}
(\byear{2008}).
\btitle{On the asymptotic properties of the group lasso estimator for linear
  models}.
\bjournal{Electron. J. Stat.}
\bvolume{2}
\bpages{605--633}.
\end{barticle}
\endbibitem

\bibitem[\protect\citeauthoryear{Poon, Keriven and
  Peyr{\'e}}{2021}]{poon2018geometry}
\begin{barticle}[author]
\bauthor{\bsnm{Poon},~\bfnm{Clarice}\binits{C.}},
  \bauthor{\bsnm{Keriven},~\bfnm{Nicolas}\binits{N.}} \AND
  \bauthor{\bsnm{Peyr{\'e}},~\bfnm{Gabriel}\binits{G.}}
(\byear{2021}).
\btitle{The geometry of off-the-grid compressed sensing}.
\bjournal{Foundations of Computational Mathematics}.
\end{barticle}
\endbibitem

\bibitem[\protect\citeauthoryear{Raskutti, Wainwright and Yu}{2011}]{MR2882274}
\begin{barticle}[author]
\bauthor{\bsnm{Raskutti},~\bfnm{Garvesh}\binits{G.}},
  \bauthor{\bsnm{Wainwright},~\bfnm{Martin~J.}\binits{M.~J.}} \AND
  \bauthor{\bsnm{Yu},~\bfnm{Bin}\binits{B.}}
(\byear{2011}).
\btitle{Minimax rates of estimation for high-dimensional linear regression over
  {$\ell_q$}-balls}.
\bjournal{IEEE Trans. Inform. Theory}
\bvolume{57}
\bpages{6976--6994}.
\end{barticle}
\endbibitem

\bibitem[\protect\citeauthoryear{Tang, Bhaskar and
  Recht}{2013}]{tang2013sparse}
\begin{binproceedings}[author]
\bauthor{\bsnm{Tang},~\bfnm{Gongguo}\binits{G.}},
  \bauthor{\bsnm{Bhaskar},~\bfnm{Badri~Narayan}\binits{B.~N.}} \AND
  \bauthor{\bsnm{Recht},~\bfnm{Benjamin}\binits{B.}}
(\byear{2013}).
\btitle{Sparse recovery over continuous dictionaries-just discretize}.
In \bbooktitle{2013 Asilomar Conference on Signals, Systems and Computers}
\bpages{1043--1047}.
\bpublisher{IEEE}.
\end{binproceedings}
\endbibitem

\bibitem[\protect\citeauthoryear{Tang, Bhaskar and Recht}{2015}]{tang2014near}
\begin{barticle}[author]
\bauthor{\bsnm{Tang},~\bfnm{Gongguo}\binits{G.}},
  \bauthor{\bsnm{Bhaskar},~\bfnm{Badri~Narayan}\binits{B.~N.}} \AND
  \bauthor{\bsnm{Recht},~\bfnm{Benjamin}\binits{B.}}
(\byear{2015}).
\btitle{Near minimax line spectral estimation}.
\bjournal{IEEE Trans. Inform. Theory}
\bvolume{61}
\bpages{499--512}.
\end{barticle}
\endbibitem

\bibitem[\protect\citeauthoryear{Tibshirani}{1996}]{tibshirani1996regression}
\begin{barticle}[author]
\bauthor{\bsnm{Tibshirani},~\bfnm{Robert}\binits{R.}}
(\byear{1996}).
\btitle{Regression shrinkage and selection via the lasso}.
\bjournal{J. Roy. Statist. Soc. Ser. B}
\bvolume{58}
\bpages{267--288}.
\end{barticle}
\endbibitem

\bibitem[\protect\citeauthoryear{van~de Geer and
  B\"{u}hlmann}{2009}]{van2009conditions}
\begin{barticle}[author]
\bauthor{\bparticle{van~de} \bsnm{Geer},~\bfnm{Sara~A.}\binits{S.~A.}} \AND
  \bauthor{\bsnm{B\"{u}hlmann},~\bfnm{Peter}\binits{P.}}
(\byear{2009}).
\btitle{On the conditions used to prove oracle results for the lasso}.
\bjournal{Electron. J. Stat.}
\bvolume{3}
\bpages{1360--1392}.
\end{barticle}
\endbibitem

\bibitem[\protect\citeauthoryear{Yuan and Lin}{2006}]{Yuan2006}
\begin{barticle}[author]
\bauthor{\bsnm{Yuan},~\bfnm{Ming}\binits{M.}} \AND
  \bauthor{\bsnm{Lin},~\bfnm{Yi}\binits{Y.}}
(\byear{2006}).
\btitle{Model selection and estimation in regression with grouped variables}.
\bjournal{J. R. Stat. Soc. Ser. B Stat. Methodol.}
\bvolume{68}
\bpages{49--67}.
\end{barticle}
\endbibitem

\end{thebibliography}

%% or include bibliography directly:
%\begin{thebibliography}{}
% \bibitem[\protect\citeauthoryear{???}{???}]{b1}
% \end{thebibliography}

%%%%%%%%%%%%%%%%%%%%%%%%%%%%%%%%%%%%%%%%%%%%%%%%%%%%%%%%%%%%%%
%%%%%%%%%%%%%%%%%%%%%%%%%%%%%%%%%%%%%%%%%%%%%%%%%%%%%

\end{document}